\documentclass[sn-mathphys-ay]{sn-jnl}

\usepackage[utf8]{inputenc}
\usepackage[english]{babel}
\usepackage[shortlabels]{enumitem}

\usepackage{mathtools,amsmath,amsthm,amssymb,bm,bbm,xfrac}

\usepackage{booktabs,multirow,multicol}
\usepackage{graphicx}
\usepackage{float,placeins}
\usepackage[size=small]{subcaption}

\usepackage{pgfplots,tikz,pgfplots}
\pgfplotsset{compat=1.16} 
\usepackage{comment}
\usepackage{color,xcolor}

\newtheorem{assumption}{Assumption}[section]
\newtheorem{lemma}{Lemma}[section]
\newtheorem{theorem}{Theorem}[section]

\usepackage[ruled,vlined]{algorithm2e}
\SetKwProg{Init}{init}{}{}
\SetKwProg{Repeat}{repeat}{}{}
\usepackage{adjustbox}

\newcommand{\blue}{\color{black}}

\begin{document}

\title[Adaptive Optimization for Prediction with Missing Data]{Adaptive Optimization for Prediction with Missing Data}


\author[1]{\fnm{Dimitris} \sur{Bertsimas}}\email{dbertsim@mit.edu}

\author[2]{\fnm{Arthur} \sur{Delarue}}\email{arthur.delarue@isye.gatech.edu}

\author[3]{\fnm{Jean} \sur{Pauphilet}}\email{jpauphilet@london.edu}

\affil[1]{\orgdiv{Sloan School of Management}, \orgname{Massachusetts Institute of Technology}, \orgaddress{\street{77 Massachusetts Ave}, \city{Cambridge}, \state{MA}, \country{USA}}}

\affil[2]{\orgdiv{H. Milton Stewart School of Industrial and Systems Engineering}, \orgname{Georgia Institute of Technology}, \orgaddress{\street{755 Ferst Dr}, \city{Atlanta}, \state{GA}, \country{USA}}}

\affil[3]{\orgname{London Business School}, \orgaddress{\street{Regent's Park}, \city{NW1 4SA London}, \country{UK}}}


\abstract{When training predictive models on data with missing entries, the most widely used and versatile approach is a pipeline technique where we first impute missing entries and then compute predictions. 
    In this paper, we view prediction with missing data as a two-stage adaptive optimization problem and propose a new class of models, \emph{adaptive} linear regression models, where the regression coefficients adapt to the set of observed features. 
    We show that some adaptive linear regression models are equivalent to  learning an imputation rule and a downstream linear regression model simultaneously instead of sequentially. We leverage this joint-impute-then-regress interpretation to generalize our framework to non-linear models. 
    In settings where data is strongly not missing at random, our methods achieve a 2--10\% improvement in out-of-sample accuracy.}




\maketitle

\section{Introduction}

Real-world datasets usually combine information from multiple sources, with different measurement units, data encoding, or structure, leading to a myriad of inconsistencies. In particular, they often come with partially observed features.
In contrast, most supervised learning models require that all input features are available for every data point.

Missing data have been studied in the statistical inference literature for several decades \citep[see, for instance,][]{little2019statistical}. The typical assumption is that data entries are missing \emph{at random} (MAR). 
Under the MAR assumption, a valid inference methodology is to \emph{impute-then-estimate}, i.e., guess the missing values as accurately as possible, then estimate the parameter of interest on the imputed dataset \citep[and even generate confidence intervals, as in][]{rubin2004multiple}. However, most inference guarantees become invalid as soon as the data is \emph{not} missing at random (NMAR); and except in some special circumstances, the validity of the MAR assumption cannot be tested from the data \citep{little1988test,jaeger2006testing}.

In practice, prediction problems with missing data are often treated as if they were inference problems. A typical approach is \emph{impute-then-regress}, where one first imputes missing values, then trains a model on the imputed dataset {\blue\citep[see][for a review]{emmanuel2021survey}}. These two models (imputation and regression) are often learned in isolation. 
Because they are designed with inference in mind, imputation methods are sound under the MAR assumption (which makes it possible to guess missing feature values from observed values of the same feature) but may not be reliable when the data is not MAR. 

{\blue Our work contributes to an active stream of research that revisits the findings of the inference literature on missing data for supervised learning tasks. 
For example, \citet{bertsimas2024simple} and \citet{Josse2019a} promote, with theoretical and empirical evidence, the use of simple imputation rules (e.g., mean imputation) for prediction, despite the fact that they are not suited for inference.
\citet{saar2007handling} provide early evidence that imputation may be sub-optimal and that training different models based on the set of features available can achieve higher performance. Their approach, which corresponds to the `fully adaptive' method in our framework, considers training a model for each possible sparsity pattern, which can be computationally prohibitive. \citet{le2020neumiss,le2021sa} propose a tailored neural network architecture to approximate this fully adaptive regressor efficiently. To the best of our knowledge, the idea of adapting the prediction rules to the set of available features has mostly been investigated for tree-based methods. Strategies implemented in popular packages include tailoring the splits to each observation in a `lazy' fashion \citep{friedman1996lazy}, using surrogate splits when information for evaluating a split is missing, or using the missingness indicator as an explicit predictive feature \citep{twala2008good}. We refer to \citet[section 6]{Josse2019a} for an extensive review and numerical comparison of these approaches. 
}

The goal of this paper is to develop two alternative approaches to natively handle missing values in prediction, while being agnostic to the (unknown) missingness mechanism. 

We first develop a novel linear regression framework in which the predictive model \emph{adapts} to the set of available features (Section \ref{sec:adaptive}). More precisely, we treat prediction with missing data as a two-stage problem in which we first observe which features are missing, then choose a predictive model to apply to the observed features. The two-stage view allows us to propose a hierarchy of adaptive models (constant, affine, polynomial, piecewise constant) following the principles of adaptive optimization. Our adaptive linear regression models extend the missing indicator method \citep{van2023missing}, which is mostly used for decision tree models \citep{twala2008good}. 

In special cases, we show that adaptive linear regression models are equivalent to an impute-then-regress strategy where the imputation and prediction models are learned jointly instead of sequentially (Section \ref{sec:extension.adaptive}). 
Based on this insight, we further extend the adaptive linear regression framework to non-linear models, in which we jointly train a simple imputation model (for missing features) and an arbitrary non-linear prediction model (for the target variable) using alternating optimization.

In experiments on synthetic (Section \ref{ssec:extension.evaluation}) and real data (Section \ref{sec:realdata}), we find that all methods perform comparably when data is missing at random but that the two algorithms we propose outperform the state-of-practice quite significantly when data is not missing at random, by up to 10\% on synthetic data and 2\% on real data. 
Since the MAR assumption cannot be tested from the data itself, the true missingness mechanism is always unknown in practice and models that enjoy robust performance irrespective of the missingness mechanism like ours are highly desirable.

\section{Linear Regression with Missing Values} \label{sec:adaptive}
In this section, we present our framework for adaptive linear regression. After introducing our generic framework in Section \ref{ssec:adaptive.framework}, we propose a hierarchy of approximations inspired by methods from the adaptive optimization literature in Section \ref{ssec:adaptive.hierarchy}. We derive finite-sample generalization bounds for models in our hierarchy in Section \ref{ssec:adaptive.generalization}.

\subsection{Model setup} \label{ssec:adaptive.framework}

We consider a prediction setting, where the goal is to predict the value of a target variable $y$ given a partially observed vector of covariates or features. Formally, we model the covariates as a random variable $\bm{X}$, and the target variable as a random variable $Y$. The random variable $\bm{M}$ is a binary vector which determines whether a particular feature is observed; $X_i$ is missing if $M_i=1$, and observed if $M_i=0$. In order to learn a predictive model, we have access to $n$ i.i.d. samples $(\bm{x}_i, \bm{m}_i)$, $i\in [n]$, where $\bm{x}_i \in \mathbb{R}^d$ is a sampled vector of covariates and $\bm{m}_i \in \{ 0,1 \}^d$ is the missingness \emph{indicator} or missingness \emph{pattern} of sample $i$. For clarity of notation, we further denote $\bm{o}(\bm{x}_i, \bm{m}_i)$ the  $(d - \| \bm{m}_i \|_0)$-dimensional vector of {\bf o}bserved covariates \citep{seaman2013meant}. Symmetrically, $\bm{o}(\bm{x}_i, \bm{1} - \bm{m}_i)$ is the vector of unobserved ones. With these notations, $(\bm{x}_i, \bm{m}_i)_{i\in [n]}$ corresponds to the dataset of \emph{realizations} while $(\bm{o}(\bm{x}_i, \bm{m}_i), \bm{m}_i)_{i\in [n]}$ is the dataset of \emph{observations}.

So far we have made no assumptions on the correlation between $\bm{X}$ and $\bm{M}$. The missing data literature distinguishes between two cases, or \emph{mechanisms}: missing at random (MAR) and not missing at random (NMAR). Formally, missing data are MAR if, conditioned on the observed covariates and target variable $(\bm{o}(\bm{X}, \bm{M}), Y)$, the missingness indicator $\bm{M}$ is independent of the unobserved covariates $\bm{o}(\bm{X}, \bm{1} - \bm{M})$. Otherwise, missing data are NMAR. The MAR assumption is necessary for the validity of standard ``impute-then-regress'' approaches, where one first trains a predictor $\mu(\cdot)$ for the missing entries $\mu(\bm{o}(\bm{x}, \bm{m}),\bm{m})\approx \mathbb{E}[o(\bm{X}, \bm{1} - \bm{M}) | \bm{M}=\bm{m}, \bm{o}(\bm{X}, \bm{M})=\bm{o}(\bm{x}, \bm{m})]$ (impute), then a second predictor for the target variable $\hat{y} = f(\bm{o}(\bm{x}, \bm{m}), \mu(\bm{o}(\bm{x}, \bm{m}),\bm{m}))$ (regress). The imputation model $\mu(\cdot)$ can only be learned under the MAR assumption. A key motivation of our work is to develop an approach that is agnostic to the missing data mechanism since it is an unverifiable property of the data.

The goal of our paper is therefore to learn $\mathbb{E}[Y | \bm{M}=\bm{m}, \bm{o}(\bm{X}, \bm{M})=\bm{o}(\bm{x}, \bm{m})]$ directly, i.e., build a predictive model that (a) applies to the observed vectors $\bm{o}(\bm{X}, \bm{M})$ directly, (b) leverages information from the missingness pattern $\bm{M}$, and (c) is agnostic to the missing data mechanism. 

Many learners, including linear models, cannot cope with the variable dimension of the vector $\bm{o}(\bm{x}, \bm{m})$ by design, hence cannot be applied to observations with missing entries directly. To address this issue, we adopt a simple convention: when computing the output of a linear model, we omit the features that are missing. Formally, given a linear model $\bm{w} \in \mathbb{R}^d$, we define (and denote) the output of the linear model $\bm{w}$ for the observation $\bm{x} \in \mathbb{R}^d$ and the missingness pattern $\bm{m} \in \{0,1\}^d$ as follows
\begin{align*}
    \langle \bm{w}, \bm{x} \rangle_{\bm{m}} := \sum_{\substack{j = 1 \\ m_j = 0}}^d w_j x_j = \sum_{j=1}^d w_j (1-m_j)x_j.
\end{align*}
Observe that the expressions above do not depend on the values of $x_j$ for the features $j$ that are missing ($m_j=1$). In other words, $\langle \bm{w}, \bm{x} \rangle_{\bm{m}}$ does not depend on the full vector $\bm{x}$ but only on its observed coordinates $\bm{o}(\bm{x},\bm{m})$. For concision, we will simply write $\langle \bm{w}, \bm{x} \rangle_{\bm{m}}$ (or later $f(\bm{x},\bm{m})$), despite the fact that these expressions technically only depend on $\bm{o}(\bm{x},\bm{m})$.
This convention is also equivalent to assuming that missing entries are systematically imputed by $0$.  

In addition, as previously discussed, we would like our predictive model to be able to leverage the information contained in the missingness pattern directly. To do so, we allow the weights of our linear models to depend explicitly on the set of features available, i.e., we consider prediction rules of the form 
\begin{align} \label{eqn:reg.full} 
    f(\bm{x}, \bm{m}) = \langle \bm{w}(\bm{m}), \bm{x} \rangle_{\bm{m}} = \sum_{j=1}^d w_j(\bm{m}) (1-m_j)x_j,
\end{align}
where the functions ${w}_j(\bm{m})$ are prescribed to a user-defined class $\mathcal{F}$. 
Given historical data $\{ (\bm{m}_i, \bm{o}(\bm{x}_i,\bm{m}_i),y_i),\ i=1,\dots,n\}$, the functions ${w}_j(\cdot)$ can be computed by solving an empirical risk minimization problem 
\begin{align} \label{eqn:adaptive.erm}
    \min_{\bm{w}(\cdot) \in \mathcal{F} } \: \dfrac{1}{n}\sum_{i=1}^n \ell \left( y_i, \langle \bm{w}(\bm{m}_i), \bm{x}_i \rangle_{\bm{m}_i} \right), 
\end{align}
for some loss function $\ell$ such as the squared error, $\ell(y,z) = (y-z)^2$.

\subsection{A hierarchy of adaptive linear regression models} \label{ssec:adaptive.hierarchy}
We now derive a hierarchy of linear models, according to how the functions $w_j(\bm{m})$ in \eqref{eqn:reg.full} depend on the missingness pattern, borrowing ideas from the multi-stage adaptive optimization literature. We first consider two extreme special cases for $\mathcal{F}$, namely the fully adaptive case and the static case. 

\paragraph{Fully adaptive regression.} The fully adaptive case corresponds to the situation where we consider a different linear model for each potential missingness patterns $\bm{m}$. Formally, each function $w_j(\bm{m})$ is of the form 
\[
    w_j(\bm{m}) = \sum_{\bm{m}' \in \{0,1\}^d} w_{j,m'} \bm{1}(\bm{m} = \bm{m}'),
\]
with
\[
    \bm{1}(\bm{m} = \bm{m}') := \begin{cases} 1 & \mbox{ if } \bm{m}=\bm{m'}, \\ 0 & \mbox{ otherwise.}\end{cases}
\]
Oberve that, since $\bm{m}$ is a binary vector, $w_j(\bm{m})$ can be considered as an order-$d$ polynomial in $\bm{m}$. Indeed, we have
\begin{align*}
  \bm{1}(\bm{m} = \bm{m}') &= \prod_{j=1}^d (1 - (m_j - m_j')^2) 
   =   \prod_{j=1}^d (1 - m_j - m_j' + 2m_j m_j').
\end{align*}
This approach is equivalent to partitioning the training dataset according to the set of available features and train a linear model on each part separately, {\blue as done in, e.g., \citet{saar2007handling}.}
There are two major drawbacks to this approach: first, it treats each missingness pattern separately, hence substantially reducing the number of observations available to fit each model. Second, there can be as many as $2^d$ potential missingness patterns, i.e, $2^d$ models to be trained, rendering the fully adaptive approach potentially intractable. Yet, often in practice, only $d' \ll d$ covariates might be subject to missingness and the actual number of patterns to consider can be substantially smaller. 

\paragraph{Static regression.} On the opposite side of the spectrum, we can consider a static model $\bm{w}(\bm{m}) = \bm{w}$ which does not depend on $\bm{m}$. This is equivalent to fitting a linear model on a full dataset where missing values are replaced by $0$. Note that this is not equivalent to applying mean-imputation on the data before training a linear model. By imputing missing values with a $0$, a missing feature effectively does not contribute to the output of the model, while any non-zero value would have affected the final prediction. 

Between these two extremes, one can consider specific functional forms for $\bm{w}(\bm{m})$ such as linear or polynomial functions that could be used to trade-off adaptivity and tractability. 

\paragraph{Affinely adaptive regression.} Affine policies are a successful and often-used tool in adaptive optimization, as they are typically more powerful than a static policy but more tractable than a fully adaptive one \citep{ben2004adjustable,bertsimas2012power}. An affinely adaptive linear model takes the form
\begin{align*}
{w}_j^{\text{affine}}(\bm{m}) = {w}_{0,j} + \sum_{j'=1}^d W_{j j'} m_{j'},
\end{align*}
or $\bm{w}_j^{\text{affine}}(\bm{m}) = \bm{w}_{0} +  \bm{W m}$ with matrix notations. Here, $\bm{w}_0$ corresponds to a baseline model to be used when all features are present and ${W}_{j j'}$ represents the linear correction to apply to the $w_{0,j}$ whenever feature $j'$ is missing. For a given observation $(\bm{o}(\bm{x},\bm{m}), \bm{m})$, a prediction is obtained by computing $\langle \bm{w}_0 + \bm{W} \bm{m}, \bm{x} \rangle_{\bm{m}} = \sum_{j}w_{0j} (1-m_j) x_j + \sum_{j,k} W_{jk} m_k (1-m_j) x_j$. Accordingly, fitting $\bm{w}_0$ and $\bm{W}$ is equivalent to fitting a linear regression model over the $d + d^2$ features of the form $(1-m_j) x_j$, for $j=1,\dots,d$ (zero-imputed dataset) and $m_k (1-m_j) x_j$ for $j,k = 1,\dots,d$. Intuitively, the model adapts to missing feature $j$ by adding a correction term to the weights of the other (observed) features.

\paragraph{Polynomially adaptive regression.} In addition to affine functions, polynomial decision rules have been theoretically and empirically investigated in the adaptive multi-stage optimization literature \citep[see, e.g.,][]{bertsimas2009hierarchy,bertsimas2011hierarchy}. Similarly, in our setting, one could consider weights that are order-$t$ polynomials, for some $t \in \{1,\dots,d\}$. In this case, each weight function $w_j(\bm{m})$ can be viewed as a linear function in the monomials of the form $\prod_{j \in \mathcal{J}} m_j$, with $\mathcal{J} \subseteq \{1,\dots,d\}$, $|\mathcal{J}| \leq t$. Hence, $f(\bm{x}, \bm{m})$ can be seen as a linear function in the $\mathcal{O}(d^{t+1})$ variables of the form $x_{j'} (1-m_{j'}) \prod_{j \in \mathcal{J}} m_j$ (and fitted as such). In this case, the correction terms to the weights of the observed features can depend on which combinations of features are missing: for example, $w_3$ might increase when either feature 1 or 2 is missing, but decrease when both are.

\paragraph{Finitely adaptive regression.} Another way to balance the trade-off between expressiveness of the adaptive model and tractability is finite adaptability. Formally, 
we can partition the space of all possible missingness patterns $\mathcal{M} \subseteq \{0,1\}^d$ into $Q$ disjoint subsets $\{\mathcal{M}_q\}_{q=1}^Q$ such that $\mathcal{M}=\cup_{q=1}^Q\mathcal{M}_q$ and define a distinct linear regression model for each $\mathcal{M}_q$, i.e., train a model of the form $\bm{w}(\bm{m}) = \sum_{q = 1}^Q \bm{w}_q \bm{1}(\bm{m} \in \mathcal{M}_q)$. As in the fully adaptive case, one can show that the indicator functions $\bm{1}(\bm{m} \in \mathcal{M}_q)$ are polynomial functions in $\bm{m}$. Hence, finitely adaptive regression can be seen as a special case of polynomially adaptive regression, where monomials are introduced dynamically.

\begin{algorithm}[ht]
\SetAlgoLined
\KwResult{Partition $\mathcal{P} = \{ \mathcal{M}_q,\ q \in [Q] \}$ and models $\{ \bm{w}^q,\ q \in [Q] \}$.}
initialization: $\mathcal{P} = \{ \mathcal{M}_1 = \mathcal{M} \}$\;
 \For{$\mathcal{M}_q \in \mathcal{P}$}{
 $\displaystyle j^\star\leftarrow \arg\min_{j, \bm{w}^{0}, \bm{w}^{1}} \: \sum_{i : \bm{m}_i \in \mathcal{M}_{q}^{j,0}} \ell \left( y_i, \langle \bm{w}^{0}, \bm{x}_i \rangle_{\bm{m}_i} \right) 
 + \sum_{i : \bm{m}_i \in \mathcal{M}_{q}^{j,1}} \ell \left( y_i, \langle \bm{w}^{1}, \bm{x}_i \rangle_{\bm{m}_i} \right)
$ \;
  \If{stopping criterion is not met}
  {split $\mathcal{M}_q$ along $j^\star$: $\mathcal{P} \leftarrow (\mathcal{P} \backslash \{ \mathcal{M}_q \}) \cup \{ \mathcal{M}^{j^\star,0}_q,\mathcal{M}^{j^\star,1}_q \} $}
 }
\caption{Iterative procedure for finitely adaptive regression} \label{alg:finite}
\end{algorithm}

The main difficulty of a finitely adaptive approach is in choosing the partition $\{\mathcal{M}_q\}_{q=1}^Q$.
We propose a greedy heuristic that simultaneously learns the partition $\{\mathcal{M}_q\}_{q=1}^Q$ and the appropriate regression models $\bm{w}_q$ based on recursive partitioning (Algorithm~\ref{alg:finite}).
At each iteration, we consider splitting each subset, denoted $\mathcal{M}_q$, into two subsets 
$\mathcal{M}_{q}^{j,0} := \mathcal{M}_{q} \cap \{\bm{m} : m_j=0\}$ and
$\mathcal{M}_{q}^{j,1}1 :=  \mathcal{M}_{q} \cap \{\bm{m} : m_j=1\}$ for all features $j$, and choose the feature $j^\star$ that leads to the highest reduction in empirical risk. 
To prevent overfitting, we add different stopping criteria that make further splitting inadmissible. 
For instance, we can impose a limit on the minimum number of samples per leaf, on the total depth of the resulting partitioning tree, or only allow splits that sufficiently reduce the in-sample error. 
{\blue We should note that the problem of finding the optimal partition $\{\mathcal{M}_q\}_{q=1}^Q$ can be formulated as a mixed-integer optimization problem. Future work could investigate the practical scalability and relevance of solving it via exact optimization methods \citep{bertsimas2017optimal} or column-generation heuristics \citep{patel2024improved} instead of our greedy procedure.}

\subsection{Finite-sample generalization bounds for adaptive linear regression} \label{ssec:adaptive.generalization}
Intuitively, choosing a more complex class of functions for our adaptive weights $w_j(\bm{m})$ can increase predictive accuracy, because it can model more complex relationships, but is also more prone to overfitting due to the increased number of parameters to calibrate. We now formalize this trade-off by deriving finite-sample bounds for adaptive linear regression models with polynomial weights. To conduct our analysis, we adopt a similar theoretical setting as \citet{le2020linear}.

\begin{assumption}\label{ass:finite.generative}We have $Y = f^\star(\bm{X}, \bm{M}) + \varepsilon$ where, conditional on $(\bm{X}, \bm{M})$, $\varepsilon$ is a centered noise with variance $\sigma^2$ and $f^\star(\bm{X}, \bm{M})$ is of the following form
\begin{align*}
    f^\star(\bm{X},\bm{M}) = \sum_{\bm{m} \in \{0,1\}^d} \left( \sum_{j=1}^d w^\star_{j,m} (1-M_j) X_j \right) \bm{1}(\bm{M} = \bm{m}).
\end{align*}
Furthermore, there exists $L >0$ such that $\| f^\star \|_\infty \leq L$. 
\end{assumption}
In the words of Section~\ref{ssec:adaptive.hierarchy}, Assumption~\ref{ass:finite.generative} states that the ground truth is a fully adaptive linear model. 
We refer to \citet[][proposition 4.1]{le2020linear} for specific conditions on the dependency between $\bm{X}$ and $Y$ and on the missingness mechanisms that lead to Assumption~\ref{ass:finite.generative}. 
{\blue In our numerical experiments (Section \ref{sec:numerics}), however, we consider a variety of settings, most of which do not satisfy this theoretical assumption.}

As in \citet[][chapter 13]{gyorfi2002distribution}, we consider a truncated version of our linear regression estimate, defined as $T_L f(x) := f(x)$ if $|f(x)| \leq L$, $L$ if $f(x) > L$, and $-L$ if $f(x) < -L$. With this assumption, we obtain the following result:

\begin{theorem}\label{prop:finite} Consider a function a class $\mathcal{F}$ and adaptive linear regression models of the form \eqref{eqn:reg.full} with $w_j(\bm{m}) \in \mathcal{F}$. Denote $\hat{f}_n$ the ordinary least square estimator obtained from a dataset of size $n$ and assume that fitting $\hat{f}_n$ is equivalent to fitting a linear model with $p(\mathcal{F})$ coefficients. Under Assumption~\ref{ass:finite.generative}, there exists a universal constant $c>0$ such that
\begin{align*}
    \mathbb{E}\left[ \left(Y - T_L \hat{f}_n(\bm{X},\bm{M})\right)^2 \right] \leq \sigma^2 + 8\, b(\mathcal{F}) + c \max\{\sigma^2, L\} \dfrac{1+\log n}{n} \, p(\mathcal{F}), 
\end{align*}
where $b(\mathcal{F})$ is an upper-bound on the estimation bias, $\min_{f \in \mathcal{F}} \mathbb{E}\left[ \left( f^\star(\bm{X}, \bm{M}) - f(\bm{X}, \bm{M}) \right)^2 \right]$ (hence depends on $\mathcal{F}$ and $d$ but is independent of $n$). Formulas/bounds for the constants $p(\mathcal{F})$ and $b(\mathcal{F})$ are given in Table~\ref{tab:finite.scaling}.
\end{theorem}

\begin{table}[ht]
    \centering
    \begin{tabular}{lcc}
        $\mathcal{F}$: class for $w_j(\bm{m})$ & $b(\mathcal{F})$ & $p(\mathcal{F})$ \\ 
        \midrule 
         Static & $\mathcal{O} \left( d \right)$ & $d$\\
         Affine & $\mathcal{O} \left( d^2 \right)$  & $d + d^2$\\
         Order-$t$ polynomial & $\mathcal{O} \left( \dfrac{d^{t+1}}{ (t+1)^{3(t+1)}} \right)$ & $\mathcal{O}(d^{t+1})$\\
         Fully adaptive & $0$ & $d \, 2^d$\\
         \bottomrule
    \end{tabular}
    \caption{Formula/bounds for the constants $b(\mathcal{F})$ and $p(\mathcal{F})$ involved in the finite-sample guarantees in Theorem~\ref{prop:finite} for different adaptive linear regression models in our hierarchy. }
    \label{tab:finite.scaling}
\end{table}

Intuitively, Theorem~\ref{prop:finite} shows that, as we consider higher-order polynomials as adaptive rules, the number of parameters to calibrate, $p(\mathcal{F})$ grows exponentially but the bias term (provided that the ground truth corresponds to a fully adaptive model as stated in Assumption~\ref{ass:finite.generative}) decays super-exponentially in $t$. {\blue Assumption \ref{ass:finite.generative} is needed in order to derive an analytical expression for (a bound on) the bias term ${b}(\mathcal{F})$. Without this assumption,  Theorem~\ref{prop:finite} remains valid with $b(\mathcal{F}) = \inf_{f \in \mathcal{F}} \mathbb{E}\left[ \left(f^\star(\bm{X},\bm{M}) - f(\bm{X},\bm{M})\right)^2 \right] $.}

{\blue In practice, we observe that the higher number of parameters $p(\mathcal{F})$ involved in polynomial functions (i.e., for models computed via Algorithm \ref{alg:finite}) leads to a higher risk of over-fitting, which in turn can deteriorate out-of-sample performance on instances with a small to moderate number of training samples (see Section \ref{ssec:extension.evaluation} and Appendix \ref{sec:add.num.syn}). On the other hand, we observe greater benefit by moving away from linear regressors, as we propose in Section~\ref{sec:extension.adaptive}.}

\subsection{Implementation details}
Our adaptive linear regression models can be viewed as linear regression models over an extended set of $p(\mathcal{F})$ features. To mitigate the issue of overfitting, we consider an $\ell_1$-$\ell_2$-penalized version of the empirical risk minimization problem \eqref{eqn:adaptive.erm}, 
\begin{align*} \blue
    \min_{\bm{w}(\cdot) \in \mathcal{F} } \: \dfrac{1}{n}\sum_{i=1}^n \ell \left( y_i, \langle \bm{w}(\bm{m}_i), \bm{x}_i \rangle_{\bm{m}_i} \right) + \sum_{k} \lambda_k \left[ \alpha |w_k| + (1-\alpha) \dfrac{w_k^2}{2} \right],
\end{align*}
commonly referred to as ElasticNet \citep[see][section 3.4.2]{friedman2001elements} and implemented in the package \texttt{glmnet}.
{\blue The parameter $\alpha$ controls the relative importance of the $\ell_1$-penalty compared with the $\ell_2^2$- or ridge regularization term. We cross-validate its value (via 5-fold cross-validation, like all other hyper-parameters in our experiments).}
{\blue A second} parameter, $\bm{\lambda}$, controls the amount of regularization.
As supported by intuition and theory, the more data is available, the less regularization is needed, {\blue so we should choose $\lambda_k \propto 1/n$}. However, in our case, many of the features are sparse binary features (e.g., $m_j$). For each of these features, the effective number of samples available to calibrate its coefficient value is not $n$ but the number of samples for which it is non-zero (e.g., $\sum_{i \in [n]} m_{ij} \leq n$). Accordingly, we apply a different $\lambda_k$ value for each $w_k$. {\blue If $w_k$ is applied to a feature of the form $x_j$ (resp. $x_j (1-m_j)$), we set $\lambda_k = \lambda /n_k$ where $n_k$ is the number of observations for which $x_j$ is available (resp. missing) and $\lambda$ is a scaling parameter we cross-validate.}

\section{Extension to Adaptive Non-Linear Regression} \label{sec:extension.adaptive}
We now extend the adaptive linear regression framework developed in Section~\ref{sec:adaptive} to generic non-linear models. To deal with inputs $\bm{o}(\bm{x},\bm{m})$ of varying dimension, we proposed in the linear case to set the missing features to zero. This convention appears natural for linear models since setting a variable to zero effectively implies that this variable does not contribute to the output. However, this convention is hard to generalize to other models (e.g., tree-based models) for which the impact of zero-imputation on the output is not predictable. To extend our adaptive regression framework, we first show a connection between adaptive linear models and impute-then-regress strategies where the imputation and regression models are learned jointly instead of sequentially (\ref{ssec:adaptive.imputation}). We refer to this family of strategies as \emph{joint} impute-then-regress, and propose a heuristic to compute them with non-linear models in Section \ref{ssec:extension.heuristic}. 

\subsection{Connection between adaptive regression and optimal impute-then-regress} \label{ssec:adaptive.imputation}
We consider a special case of affinely adaptive linear regression where all the regression coefficients are static except for the intercept (i.e., a feature $x_j$ that is constant, equal to $1$, and never missing), which depends on the missingness pattern in an affine way, i.e.,
$f(\bm{x}, \bm{m}) = b(\bm{m}) + \langle \bm{w}, \bm{x} \rangle_{\bm{m}}$,
where $b(\bm{m}) = b_0 + \sum_j b_j m_j$. In this case, the prediction function is 
\begin{align*}
f(\bm{x}, \bm{m}) &= b_0 + \sum_{j=1}^d \left( {w}_j  (1-{m}_j) {x}_j + b_j {m}_j \right)\\ &= b_0 + \sum_{j=1}^d w_j \left( (1-{m}_j) {x}_j + {m}_j  \dfrac{b_j}{w_j} \right).
\end{align*}
In other words, a static regression model with affinely adaptive intercept can be viewed as imputing $\mu_j := b_j / w_j$ for feature $j$ whenever it is missing, and then applying a linear model $\bm{w}$. The key difference with standard impute-then-regress strategies, however, is that the vector of imputed values $\bm{\mu}$ and the linear model $\bm{w}$ are computed simultaneously, instead of sequentially, hence leading to greater predictive power\footnote{Note that our algebraic manipulation, and the resulting interpretation, is valid only if $w_j \neq 0$. If $w_j = 0$ and $b_j \neq 0$, it means that feature $j$ is not a strong predictor of the outcome variable $y$ but the fact that it is missing, $m_j$, is.}. Note that this family of models correspond to the affine approximation in \citet[][definition 4.1]{le2020linear}.

In the simple case where there is only one feature $X_1$, this family of models would learn the rule $w_1 (X_1 (1-M_1) + \mu_1 M_1)$ with 
\begin{align*}
    \mu_1 = \dfrac{1}{w_1} \mathbb{E}[Y | M_1 = 1] = \dfrac{\mathbb{E}[Y | M_1 = 1] \mathbb{E}[X_1^2 | M_1 = 0]}{\mathbb{E}[Y X_1 | M_1 = 0]}.
\end{align*}
Compared with classical imputation methods, we observe that the imputed value does not only depend on the distribution of $X_1$ on the samples where it is observed ($M_1=0$). Rather, (a) it depends on the target variable $Y$, and (b) it involves observations for which $X_1$ is missing ($M_1=1$). In particular, if $Y$ satisfies a linear relationship $Y = w_1^\star X_1$, then $\mu_1 = \mathbb{E}[X_1 | M_1 = 1]$. In contrast, standard mean-imputation would select $\mu_1=\mathbb{E}[X_1 | M_1=0]$. We can interpret adaptive linear regression models as jointly learning an imputation rule and a predictive model, where more complex adaptive rules for the intercept can be interpreted as more sophisticated imputation rules. Jointly performing linear regression while imputing each feature by a constant is equivalent to a static regression model with an affinely adaptive intercept; such a model is very easy to compute in practice, corresponding to a linear regression model over $2d +1$ variables: the $d$ coordinates of $\bm{x}$ (with missing values imputed as 0), the $d$ coordinates of $\bm{m}$, and an intercept term $b_0$.

More generally, imputation and regression try to achieve different objectives (respectively predicting the missing entries and the target outcome) when they are considered sequentially. Considering them jointly allows imputation to account for downstream predictive accuracy. This echoes the smart predict-then-optimize approach of \citet{elmachtoub2022smart} that incorporates the downstream optimization goal into the training of the upstream predictive model.
We will use this interpretation to generalize our adaptive linear regression framework to non-linear predictors in the following section.

\subsection{Heuristic method for joint impute-and-regress} \label{ssec:extension.heuristic}
We now propose an iterative heuristic to jointly optimize for a simple imputation model $\bm{\mu}$ and a downstream predictive model $f$. 
{\blue Formally, the joint impute-then-regress problem solves the following empirical risk minimization problem
\begin{align} \label{eqn:jitr.erm}
    \min_{\bm{\mu} \in \mathbb{R}^d} \: \min_{f(\cdot)} \: \dfrac{1}{n}\sum_{i=1}^n \ell \left( y_i, f(\bm{x}^{\bm{\mu}}_i) \right), 
\end{align}
where both the imputed values $\bm{\mu}$ and the regressor $f (\cdot)$ are calibrated from data, with the objective to minimize in-sample error. Note that $\bm{x}^{\bm \mu}$ is linear in $\bm{\mu}$ as $x^\mu_j = x_j + m_j \mu_j$. However, when $f(\cdot)$ is an arbitrary non-linear function, such as the output layer from a deep neural network, the optimization problem \eqref{eqn:jitr.erm} is non-convex. The objective function in \eqref{eqn:jitr.erm} is obviously convex in $(f,\bm{\mu})$ only in the case when $f(\cdot) = \bm{w}^\top \cdot$ is linear (i.e., Section \ref{sec:adaptive}). Actually, when $f(\cdot)$ is non-linear, optimizing for $\bm{\mu}$ with $f(\cdot)$ fixed is not trivial either. To offer a simple and practical solution to this problem, we propose an alternating minimization procedure in Algorithm~\ref{alg:joint.itr}. 
}

{\blue Motivated by the strong empirical performance of mean imputation \citep{bertsimas2024simple},} we initialize $\bm{\mu}$ with the mean of each variable. 
At each iteration, we train the predictive model on the $\bm{\mu}$-imputed data, $\bm{X}^{\bm{\mu}}${\blue ---nowadays, open-source and efficient implementations algorithms for doing so are available for a wide range of problem classes}. Then, for a fixed model $f$, we update the vector of imputed values $\bm{\mu}$ so as to decrease the prediction $\operatorname{error}$ (e.g., mean squared error for regression, AUC for classification). {\blue As discussed earlier,} 
minimizing $\bm{\mu} \mapsto \operatorname{error}\left(Y, f \left(\bm{X}^{\mu}\right) \right)$ may not be tractable {\blue because $f$ is non-linear}, so we use a local search heuristic (or cyclic coordinate update heuristic) instead: 
For each coordinate $j \in [p]$, we {\blue myopically} update $\mu_j$ by increments of $\pm \sigma_j$ where $\sigma_j$ is the standard error on the mean of $X_j$, {\blue and stop after a fixed number of iterations (in our implementation, 20) or if no local change in $\bm{\mu}$ improves prediction error.} 
{\blue We terminate the outer loop in Algorithm~\ref{alg:joint.itr} after a fixed}
number of iterations (in our implementation, 10)
or when the relative improvement in prediction error {\blue falls below a user-specified threshold} ($10^{-4}$).

\begin{algorithm}[ht]
\SetAlgoLined
\KwResult{Vector of imputed values $\bm{\mu}$ and predictive model $f$.}
\Init{}{
$\mu_j \leftarrow \operatorname{mean}(X_j | M_j = 0)$\;
$\sigma_j \leftarrow \operatorname{std}(X_j) / \sqrt{n} $\;
}
\Repeat{}{
    $f \leftarrow $ best model to predict $Y$ given $\bm{X}^{\mu}$ \;
    \Repeat{}{
        \For{$j = 1,\dots,p$}{
            $\epsilon_j \leftarrow \arg\min_{\epsilon \in \{-1,0,1\}} \: \operatorname{error}\left(Y, f \left(\bm{X}^{\mu + \epsilon \sigma_j \bm{e}_j}\right)\right)$\;
            $\mu_j \leftarrow \mu_j + \epsilon_j \sigma_j$
    }
 }
}
\caption{Heuristic iterative procedure for joint $\bm{\mu}$-impute-then-regress} \label{alg:joint.itr}
\end{algorithm}

{\blue We acknowledge that the update of $\bm{\mu}$ in our algorithm follows a heuristic, local-search process that could be improved in future work, e.g., by investigating the potential applicability of non-convex numerical optimization algorithms \citep[in the spirit of, e.g.,][]{ghanbari2017black}. However, given that our main objective is to demonstrate the practical value of jointly estimating $\bm{\mu}$ and $f$, we decide to keep Algorithm~\ref{alg:joint.itr} simple and versatile as a first proof of concept.}

{\blue At a high-level, our alternating minimization procedure resembles the Expectation-Maximization (EM) algorithm. The main difference between our approach and the EM algorithm resides in the objective function of our optimization problem \eqref{eqn:jitr.erm}.  We usually refer to the EM algorithm in the context of maximum likelihood estimation. In this context, we need to assume that the data ---$(\bm{X},Y)$ in supervised learning--- is generated according to some distribution. We calibrate this distribution from data, to maximize the likelihood of the observations available. In turn, this distributional model can be used to make predictions on $Y$ given $\bm{X}$. The EM algorithm is used to estimate such models in presence of missing data: At each iteration, the missing values are imputed and replaced by their most likely values (according to the assumed distribution) and then the estimate of the data generating density is updated \citep[see, e.g.,][in logistic regression]{jiang2020logistic}. Because these approaches make assumptions on the data generating process, they implicitly require the MAR assumption (because the generating distribution has to be the same across observations with and without missing entries) and the imputed values are by definition `likely' (i.e., realistic of what the true value might have been). On the other hand, our approach falls under the empirical risk minimization paradigm. The regressor $f(\cdot)$ does not necessarily rely on assumptions on the data generating process, and the imputation rule $\bm{\mu}$ is designed to minimize prediction error, instead of maximizing an arbitrary likelihood function.}

\section{Numerical Experiments}\label{sec:numerics}

We now present a thorough numerical evaluation of our proposed adaptive regression and joint-impute-then-regress methods. Evaluating the performance of missing data methods can be tricky. Real datasets with missing values do not come with counterfactuals; meanwhile, synthetic datasets may not properly describe how data normally go missing. Therefore, in this section we benchmark our methods in a variety of settings, including synthetic and real data, to verify the robustness of their performance.

\subsection{Two experimental settings} \label{ssec:numerical-setup}

We consider two settings, one where the data is fully synthetic, and one where the design matrix $\bm{X}$ is taken from openly available datasets (for which we consider both synthetic and real signals $Y$). We provide a short overview of these setting here, and refer the reader to Appendix~\ref{sec:data} for further details.

In the {synthetic} setting, we follow a common approach in the literature \citep[see, e.g., ][section 7]{le2020linear} and sample each feature vector $\bm{x}$ from a multivariate Gaussian distribution. We then generate a (continuous) dependent variable $Y$ as a linear or nonlinear (neural network) function of $\bm{x}$. Finally, we impute data, i.e., generate missingness patterns, either completely at random (MCAR) or by censoring the extreme values (an extreme case of NMAR). Overall, we control the number of observations in our dataset ($n$, ranging from 40 to 1,000), the relationship between $Y$ and $\bm{X}$ (two alternatives), the missingness mechanisms (two mechanisms), and the fraction of missing entries ($p$, ranging from 0.1 to 0.8), leading to $49 \times 2 \times 2 \times 8 = 1,568$ different instances. {\blue Out of these four configurations ($Y$ linear/non-linear, MCAR/censoring), only one ($Y$ linear, MCAR) satisfies Assumption \ref{ass:finite.generative}.}

To complement these synthetic instances, we assemble a corpus of 63 publicly available datasets with missing data, from the UCI Machine Learning Repository and the RDatasets Repository. For these datasets, we consider both real-world and synthetic signals. More precisely, 46 of the 63 datasets identify a clear target outcome (or dependent variable)---we refer to this case as the `Real' setting (see Figure~\ref{fig:exp.designs}). 
In addition, we also create instances where the dependent variable is synthetically generated from the real feature vectors. To do so, we apply the following methodology for each dataset: First, we generate a fully imputed version of the dataset. Then, we generate a synthetic signal, as a function (again, linear or non-linear) of (a) the (imputed) features values only (Syn-MAR), or (b) of both the feature values and the missingness indicator (Syn-NMAR). We also consider an adversarial setting (c) where the dependent variable is generated as a function of the imputed feature values as in (a), but where we adversarially reallocate the missingness patterns across observations to ensure the data is not missing at random (Syn-AM). We graphically summarize the four experimental designs we construct from these real datasets in Figure \ref{fig:exp.designs} and provide a more detailed description of the datasets and the methodology in Section \ref{ssec:data.realx}.

\begin{figure}[ht]
    \centering
    \begin{subfigure}{.24\textwidth}
    \centering
\begin{tikzpicture}
\node [circle,thick,draw] (X) {$\bm{X}$};
\path (X) ++(-90:1in) node [circle,thick,draw] (M) {$\bm{M}$};
\path (X) ++(-30:1in) node [circle,thick,draw] (Y) {$Y$};

\draw[->] (X) -- (Y);
\draw[dashed] (X) -- (M) node [left=5pt,pos=0.5] {?} ;
\end{tikzpicture}
    \caption{Syn-MAR}
    \end{subfigure}
    \begin{subfigure}{.24\textwidth}
    \centering
\begin{tikzpicture}
\node [circle,thick,draw] (X) {$\bm{X}$};
\path (X) ++(-90:1in) node [circle,thick,draw] (M) {$\bm{M}$};
\path (X) ++(-30:1in) node [circle,thick,draw] (Y) {$Y$};
\draw[->] (X) -- (Y);
\draw[->] (M) -- (Y);
\draw[dashed] (X) -- (M) node [left=5pt,pos=0.5] {?} ;
\end{tikzpicture}
    \caption{Syn-NMAR}
    \end{subfigure}
    \begin{subfigure}{.24\textwidth}
    \centering
\begin{tikzpicture}
\node [circle,thick,draw] (X) {$\bm{X}$};
\path (X) ++(-90:1in) node [circle,thick,draw] (M) {$\bm{M}$};
\path (X) ++(-30:1in) node [circle,thick,draw] (Y) {$Y$};

\draw[->] (X) -- (Y);
\draw[->] (X) -- (M);
\end{tikzpicture}
    \caption{Syn-AM}
    \end{subfigure}
    \begin{subfigure}{.24\textwidth}
    \centering
\begin{tikzpicture}
\node [circle,thick,draw] (X) {$\bm{X}$};
\path (X) ++(-90:1in) node [circle,thick,draw] (M) {$\bm{M}$};
\path (X) ++(-30:1in) node [circle,thick,draw] (Y) {$Y$};
\draw[dashed] (X) -- (Y) node [above=5pt,pos=0.5] {?};
\draw[dashed] (X) -- (M) node [left=5pt,pos=0.5] {?};
\draw[dashed] (Y) -- (M) node [below=5pt,pos=0.5] {?} ;
\end{tikzpicture}
    \caption{Real}
    \end{subfigure}
    \caption{Graphical representation of the 4 experimental designs implemented in our benchmark simulations with real-world design matrix $\bm{X}$. Solid (resp. dashed) lines correspond to correlations explicitly (resp. not explicitly) controlled in our experiments.}
    \label{fig:exp.designs}
\end{figure}

\subsection{Experiments on synthetic data} \label{ssec:extension.evaluation}
As described in Section~\ref{ssec:numerical-setup}, we generate instances with 2 different models for the relationship between $Y$ and $\bm{X}$ (linear or neural network) and 2 missingness mechanisms (missing completely at random and censoring). We also vary the sample size and the proportion of missing entries. 

We first compare the performance of our heuristic for joint impute-then-regress presented in Section~\ref{ssec:extension.heuristic}, with that of mean impute-then-regress and of random forest (RF) trained directly on missing data models using the ``Missing Incorporated in Attribute'' (MIA) method of \citet{twala2008good}; see \citet[][remark 5]{Josse2019a} for its implementation.  For the first two methods,  
the downstream predictive model is  chosen via 5-fold cross-validation among linear, tree, random forest, and XGBoost (`best'). 

Figure~\ref{fig:syn.ylin} compares the out-of-sample $R^2$ of these methods in settings where $Y$ is generated according to a linear model and data is missing completely at random (resp. censored) on the top (resp. bottom) panel. 
Since $Y$ is generated as a linear function of $\bm{X}$, we also compare with the optimal $\bm{\mu}$-impute-then-linearly-regress model from Section~\ref{ssec:adaptive.imputation}, which we expect to be particularly accurate in this setting. {\blue We make the following observations:}
\begin{itemize}
    \item The performance of our joint impute-then-regress heuristic is comparable with that of optimal $\bm{\mu}$-impute-then-linearly-regress, suggesting that our heuristic is effective is jointly optimizing for the imputation and the regression model. 
In addition, our heuristic is versatile in that it can be used with any predictive model, including non-linear ones.

\item While competitive in MAR settings, mean-impute-then-regress provides suboptimal predictive power whenever the data is censored. In particular, adaptive linear regression and joint impute-then-regress methods achieve a significantly higher out-of-sample accuracy, of up to 10\%. These results highlight the key deficiency of naive impute-then-regress strategies: they hide the missingness indicator to the downstream predictive model, hence losing potentially predictive information. On the contrary, optimal $\bm{\mu}$-impute-then-linearly-regress and our impute-then-regress heuristic are the best performing methods, in both MAR and censored settings. 

\item On these instances where $Y$ is a linear function of the features, we observe that RF with MIA encoding performs strictly worse than all other methods in the MCAR setting. When data is censored, however, it provides a moderate improvement over mean impute-then-regress, but remains dominated by both our methods, optimal $\bm{\mu}$-impute-then-regress and joint impute-then-regress. 
\end{itemize}
\begin{figure}
    \begin{subfigure}[t]{\columnwidth}
        \centering
        \includegraphics[width=.65\columnwidth]{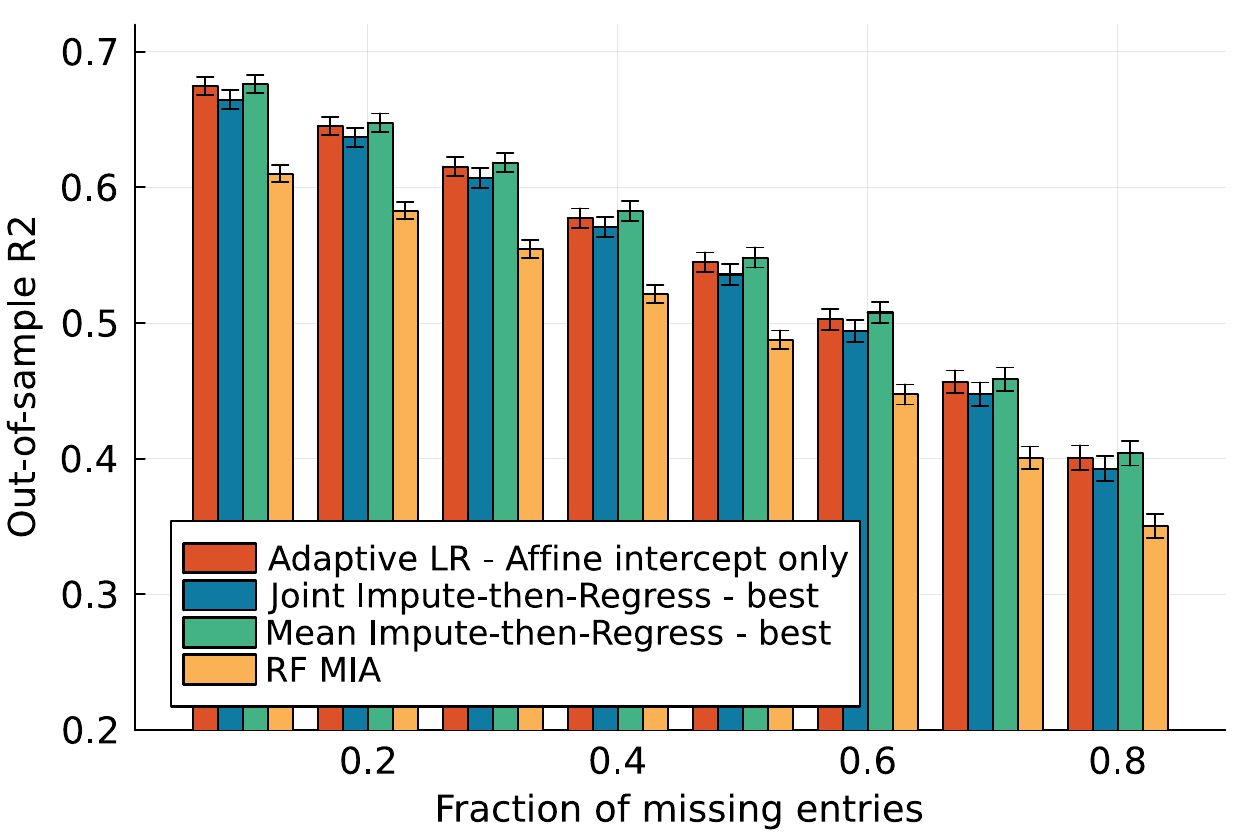}
        \caption{MCAR}
    \end{subfigure} %
    \begin{subfigure}[t]{\columnwidth}
        \centering
        \includegraphics[width=.65\columnwidth]{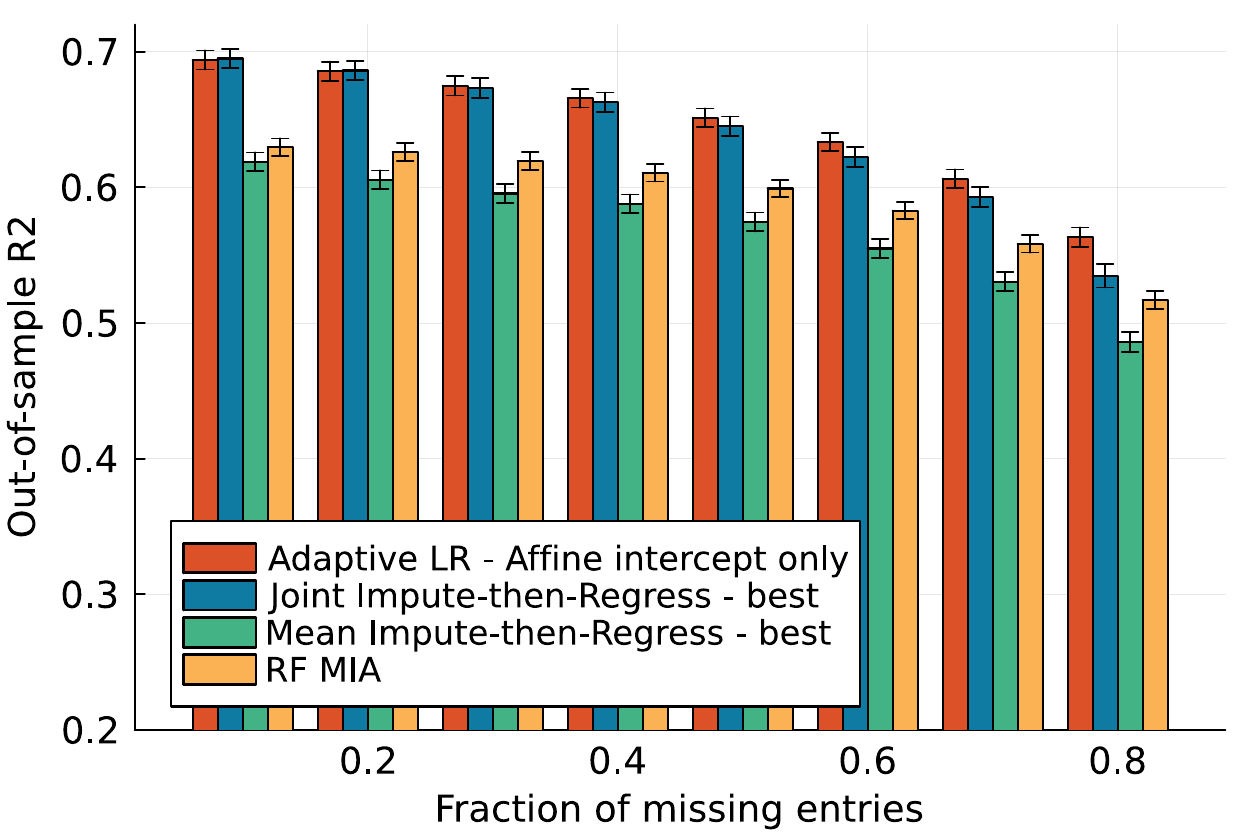}
        \caption{Censoring}
    \end{subfigure} %
    \caption{Out-of-sample $R^2$ of mean-impute-then-regress, joint impute-then-regress, and optimal $\bm{\mu}$-impute-then-regress (synthetic data, linear signal).}
    \label{fig:syn.ylin}
\end{figure}

Our observations are largely confirmed on instances where $Y$ is generated according to a neural network model {\blue (see Figure~\ref{fig:syn.ynn} in Appendix \ref{sec:add.num.syn})}. 
Fixing the fraction of missing entries to 0.3 but stratifying the results by sample size $n$, we can also confirm that our conclusions are consistent across all values of $n$, ranging from $n=40$ to $n=1,000$ (Figure \ref{fig:syn.n.ylin}).
\begin{figure}
    \begin{subfigure}[t]{\columnwidth}
        \centering
        \includegraphics[width=.65\columnwidth]{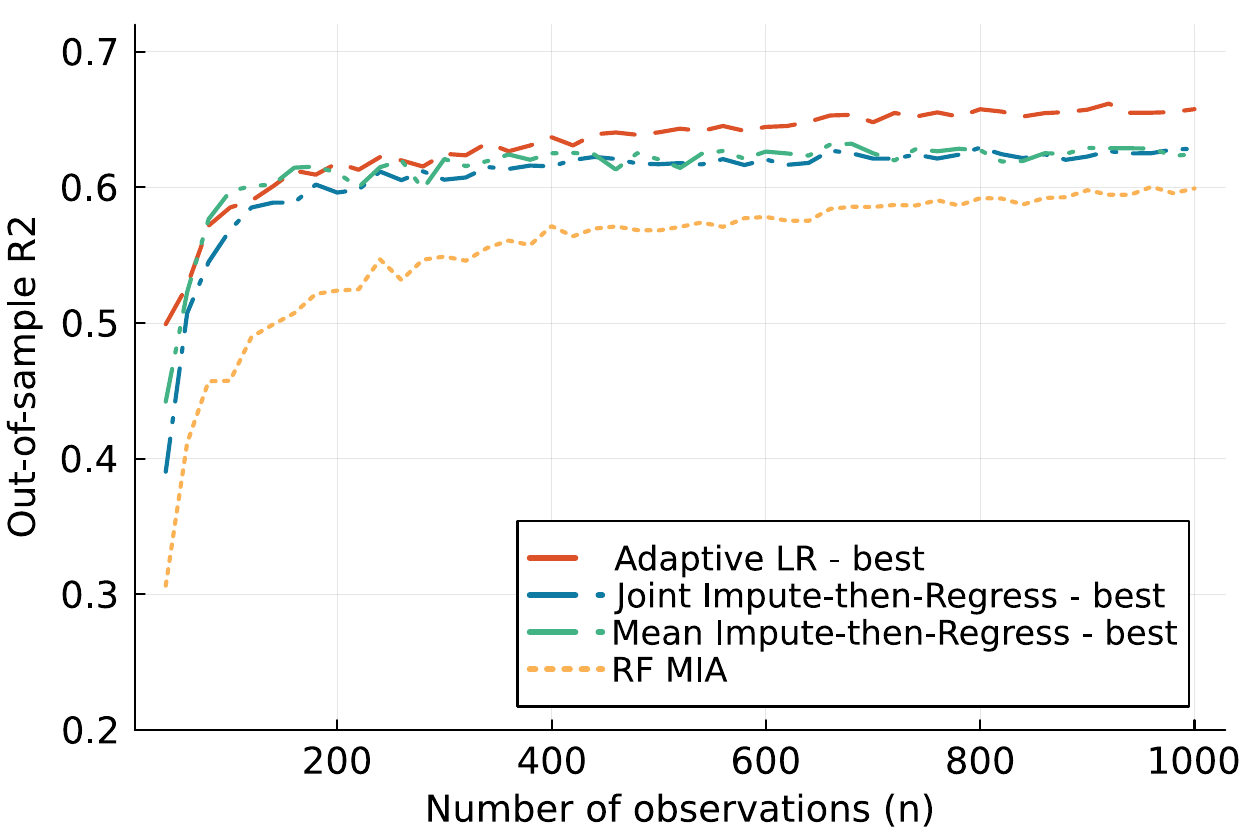}
        \caption{MCAR}
    \end{subfigure} %
    \begin{subfigure}[t]{\columnwidth}
        \centering
        \includegraphics[width=.65\columnwidth]{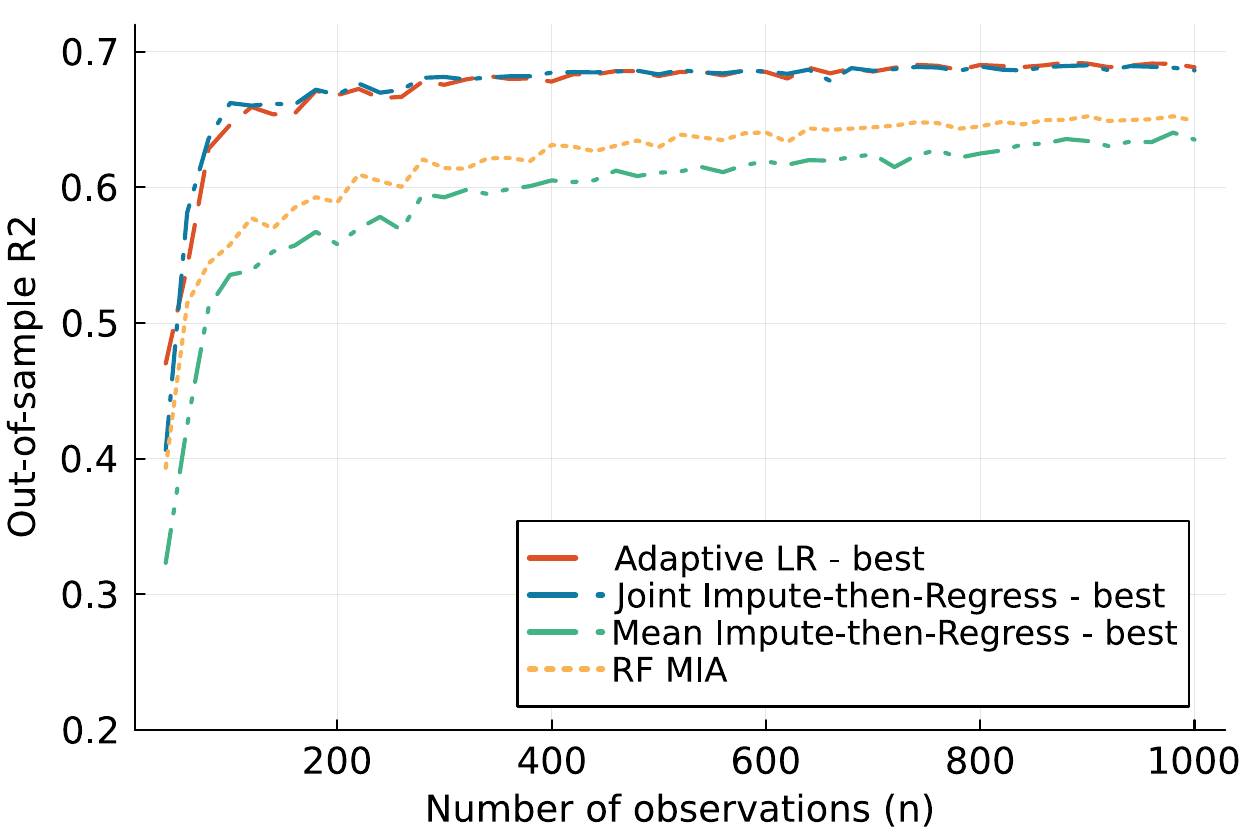}
        \caption{Censoring}
    \end{subfigure} %
    \caption{\blue Out-of-sample $R^2$ of mean-impute-then-regress, joint impute-then-regress, and optimal $\bm{\mu}$-impute-then-regress (synthetic data, linear signal), as the number of observations $n$ increases.}
    \label{fig:syn.n.ylin}
\end{figure}

\FloatBarrier
To summarize and complement these results, Table~\ref{tab:syn.res.table} reports the performance of all the methods proposed in this paper with several benchmarks, in these four synthetic-data settings. Namely, we report the performance of the best adaptive linear regression model (Section~\ref{sec:adaptive}), chosen among static with affine intercept, affine, or finite models via 5-fold cross validation on the training set (`Adaptive LR - best') and of our joint impute-then-regress (with the downstream predictor chosen among linear, tree, random forest, and XGBoost via 5-fold cross-validation on the training set as well). {\blue For adaptive LR, we report the performance of each adaptivity level separately in Table \ref{tab:syn.res.table.alr}.} 
In terms of benchmark, we compare with two impute-then-regress strategies, one with mean imputation and another one with a {\blue more} sophisticated imputation rule called \texttt{mice} \citep{buuren2010mice}. Regarding methods that can be trained on missing data directly, we compare with CART and random forest (RF) models with MIA encoding, 
and with XGBoost.
We still observe that sequential impute-then-regress strategies are mostly competitive when the data is MCAR but their relative performance deteriorates when data is not MAR. On the contrary, methods that can be trained on missing data directly tend to under-perform when MCAR and become more relevant on censored data. In contrast, our adaptive linear regression and joint impute-then-regress models are among the best performing methods irrespective of the missing data mechanism, and dominate all the benchmarks when data is censored. We also make three striking observations: (i) when the data is not missing at random,  the predictive power (as captured by $R^2$) can be higher than in the MCAR case; (ii) our adaptive linear regression models, although linear, perform very well in our experiments, even when the true signal is generated by a neural network; (iii) mean impute-then-regress constitutes a surprisingly strong benchmark, despite the simplicity of the imputation rules, thanks to the (potential) complexity of the downstream prediction rule \citep{bertsimas2024simple}.

\begin{table}
    \centering \footnotesize
    \begin{tabular}{l|cc|cc}
        Method & \multicolumn{2}{c}{MCAR} &  \multicolumn{2}{c}{Censoring} \\
        & Linear & NN & Linear & NN \\
\midrule
Adaptive LR - best & \bf 0.558 (0.003) & 0.354 (0.004) & \bf 0.641 (0.003) & \bf 0.524 (0.027)  \\ 
Joint Impute-then-Regress - best & 0.537 (0.003) & 0.389 (0.007) & 0.631 (0.003) & 0.514 (0.004)  \\ 
\midrule
Mean Impute-then-Regress - best & 0.533 (0.014) & 0.397 (0.003) & 0.562 (0.003) & 0.469 (0.003)  \\ 
mice Impute-then-Regress - best & 0.556 (0.004) & \bf 0.424 (0.004) & 0.457 (0.004) & 0.340 (0.004)  \\ \midrule
CART MIA & 0.321 (0.003) & 0.222 (0.003) & 0.451 (0.003) & 0.359 (0.003)   \\ 
RF MIA & 0.488 (0.003) & 0.397 (0.003) & 0.587 (0.002) & 0.515 (0.003)   \\ 
XGBoost & 0.474 (0.003) & 0.368 (0.003) & 0.580 (0.003) & 0.495 (0.003)   \\ 

    \end{tabular}
    \caption{Average out-of-sample $R^2$ (and standard error) for each method on synthetic datasets, where $Y$ is generated according to a linear or a neural network model, and the data is either missing completely at random or censored. Results averaged over $50$ training sizes, $8$ missingness levels, and $10$ random training/test splits. }
    \label{tab:syn.res.table}
\end{table}

\FloatBarrier
\subsection{Experiments on real data} \label{sec:realdata}
We now evaluate the performance of our methods on real-world data. 
Our objective is twofold: to confirm our findings on synthetic data from the previous section, and to appreciate the differences between the fully synthetic setting and real data.
As described in Section~\ref{ssec:numerical-setup}, our experimental benchmark comprises both synthetic (linear and neural network) signals and real signals. For synthetic signals, we control the dependency between $Y$ and $(\bm{X},\bm{M})$ so that features can be missing at random (MAR), not missing at random (NMAR), or adversarially missing (AM), as explained in Appendix~\ref{sssec:data.realx.syny}. 

Figure~\ref{fig:validation.ourmethods} displays the average out-of-sample performance for our adaptive linear regression models (top panel) and our joint impute-then-regress heuristic (bottom panel). For adaptive linear regression (Figure \ref{fig:validation.alr}), we observe that the static with affine intercept and affine variants achieve the best and comparable performance, while finite adaptive is often the worst performing method.
{\blue However, for tractability considerations (see Figure \ref{fig:validation.alr.time}), we would recommend to prioritize the affine intercept model over the fully affine one.}
By cross-validating the degree of adaptivity, the `best' variant successfully achieves the best performance overall. 
For the joint impute-then-regress heuristic (Figure \ref{fig:validation.jitr}), we surprisingly observe very strong performance from using a linear downstream predictive model, especially on the instances with synthetic signals, and even when the true signal is non-linear. Using random forest or XGBoost as regressors achieve the second best accuracy on the synthetic signal instances, while being the most accurate on the real signal ones. {\blue Training XGBoost, however, is in general faster than random forest (see Figure \ref{fig:validation.jitr.time}). As a result, we propose to cross-validate the downstream predictor between the linear and XGBoost model only, to avoid a computationally prohibitive search (we implement this cross-validated predictor in the `best' variant reported in Figure \ref{fig:validation.jitr}).} Overall, our proposal cross-validating the downstream predictor is effective in achieving the best performance. Finally, we should note that using a simple tree model as the downstream predictor is the worst performing model, by a significant margin, which suggests that our heuristic procedure might be more favorable to models whose output depends more continuously on the input variables.

\begin{figure}
    \begin{subfigure}[t]{\columnwidth}
        \centering
        \includegraphics[width=.65\columnwidth]{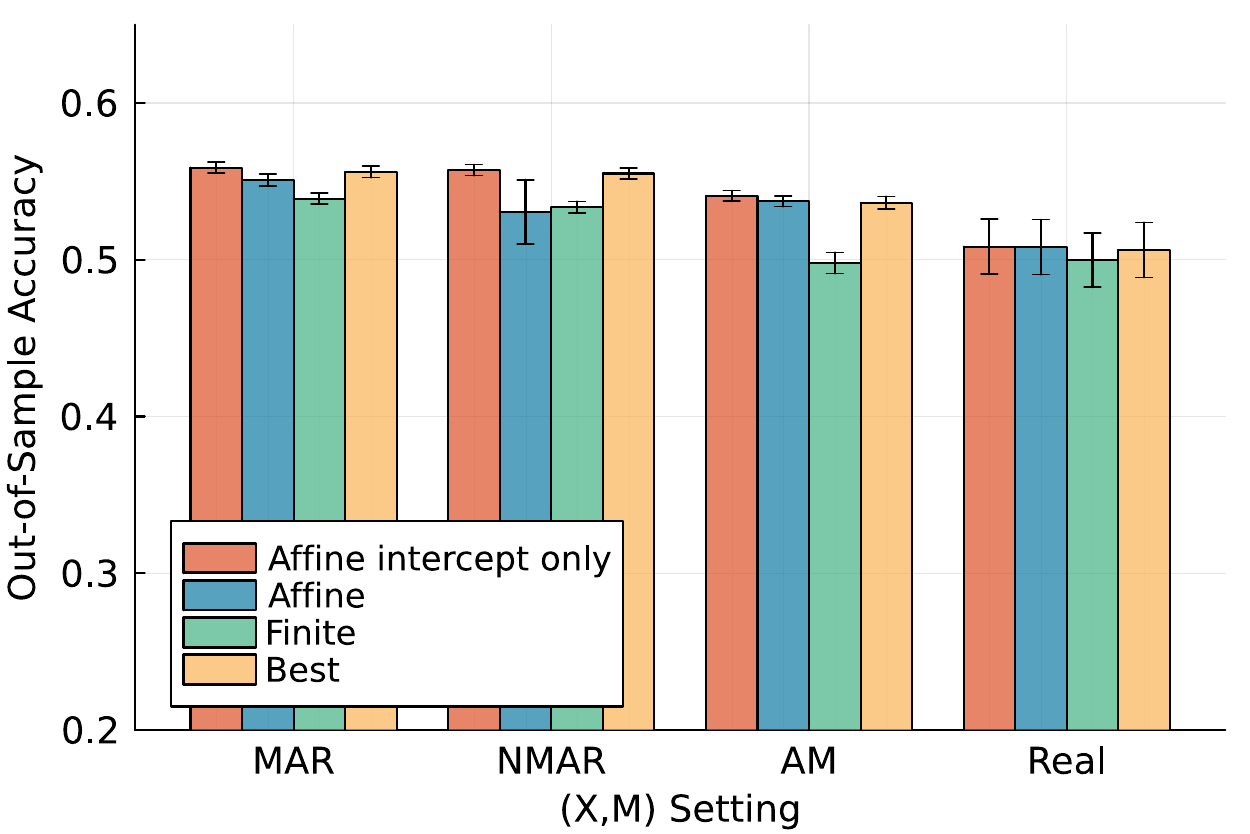}
        \caption{Adaptive Linear Regression} \label{fig:validation.alr}
    \end{subfigure} %
    \begin{subfigure}[t]{\columnwidth}
        \centering
        \includegraphics[width=.65\columnwidth]{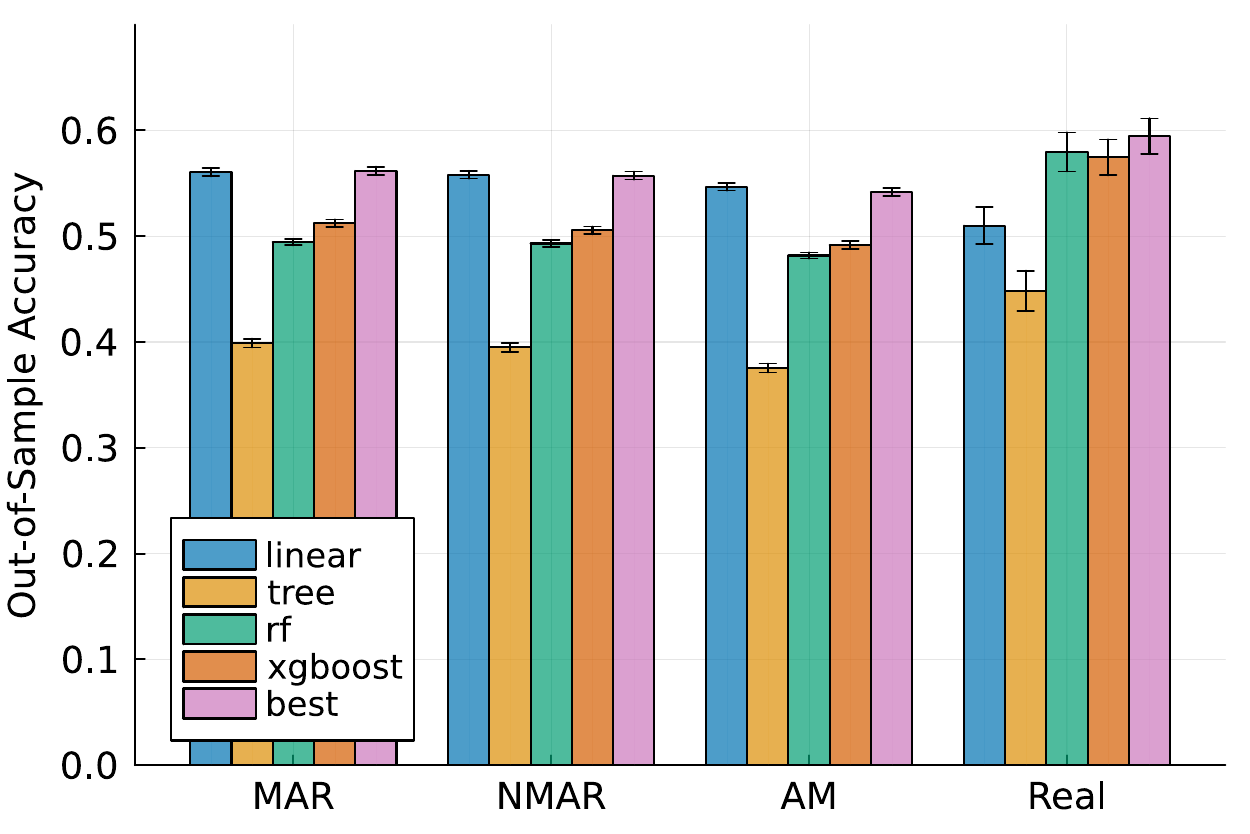}
        \caption{Joint Impute-then-Regress}\label{fig:validation.jitr}
    \end{subfigure} %
    \caption{Performance of adaptive linear regression (a) and joint impute-then-regress models (b) on real-world design matrix $\bm{X}$.}
    \label{fig:validation.ourmethods}
\end{figure}

We now compare the accuracy of the adaptive linear regression {\blue (affine intercept)} and joint impute-then-regress model {\blue (best between linear and XGBoost)} with mean-impute-then-regress and \texttt{mice}-impute-then-regress models (with the downstream predictive model chosen via cross-validation), as well as random forest and XGBoost models trained on missing data directly. We also compare to a ``complete feature'' regression model (i.e., using only features that are never missing) and, for synthetic signals, an ``oracle'' model, which has access to the fully observed dataset $\bm{x}_{full}$ and $\bm{m}$. As displayed in Figure~\ref{fig:realx.comparison}, we first observe that the complete feature regression model performs much worse than all other methods, emphasizing the danger of systematically discarding missing features in practice. Note that this conclusion holds despite the fact that features are correlated. Second, while competitive under MAR, the performance of mean impute-then-regress deteriorates as missingness becomes adversarially missing (AM), compared with both the oracle and our adaptive/joint impute-then-regress strategies. We should acknowlege, though, that the deterioration is less acute than on synthetic data. Again, these results highlight that applying imputation and regression sequentially cannot appropriately leverage the predictive power of missingness itself. Nonetheless, we should mention that on real signals $Y$, mean impute-then-regress is the best performing method, ex-aequo with joint impute-then-regress, which could be explained by the fact that the real-world data is closer to MAR than AM. 
\begin{figure}
    \centering
    \includegraphics[width=\columnwidth]{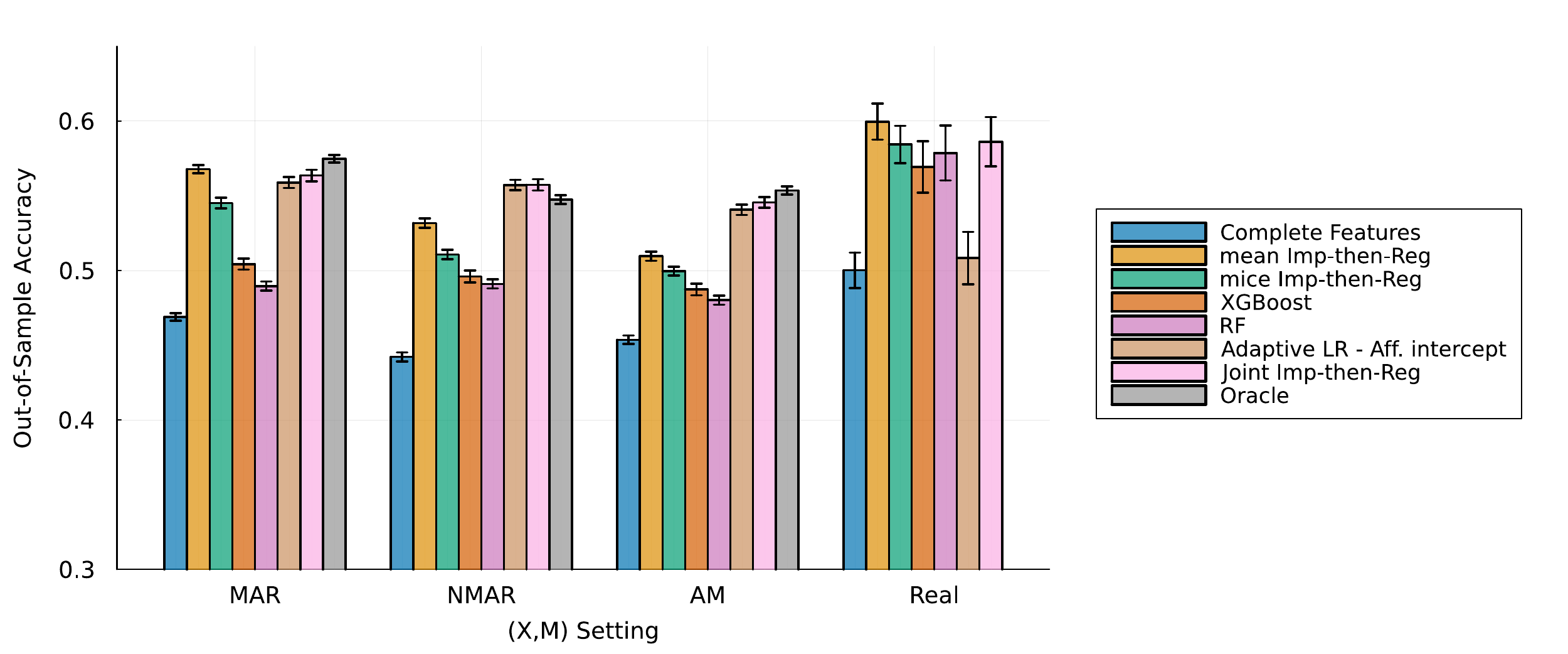}
    \caption{Comparison of adaptive linear regression and joint impute-then-regress methods vs. mean-impute-then-regress on real-world design matrix $\bm{X}$. We also compare to the ``Oracle'' which has access to the fully observed data and the ``Complete Feature'' regression which only uses features which are never missing.}
    \label{fig:realx.comparison}
\end{figure}
Note that the reported accuracy for real signals {\blue in Figure \ref{fig:realx.comparison}} aggregates the $R^2$ for regression tasks and $2 (AUC - 0.5)$ for classification. Figure~\ref{fig:realx.realy} in Appendix~\ref{sec:add.num} reports the results for regression and classification separately. For synthetic signals, we further stratify the results by the fraction of missing features contributing to the signal in Figure~\ref{fig:adaptive.all}. 

{\blue We also report computational times in Figure \ref{fig:realx.comparison.time}. Overall, we observe that adaptive linear regression with affine intercept is as computationally expensive as as mean- or mice-impute-then-regress on datasets with synthetic signals, while joint impute-then-regress can require up to twice as much time. Surprisingly, the computational times for the real signals $Y$ are significantly different, suggesting again that real data may follow patterns and structure that are not appropriately captured in our data generation strategies. 
\begin{figure}
    \centering
    \includegraphics[width=\columnwidth]{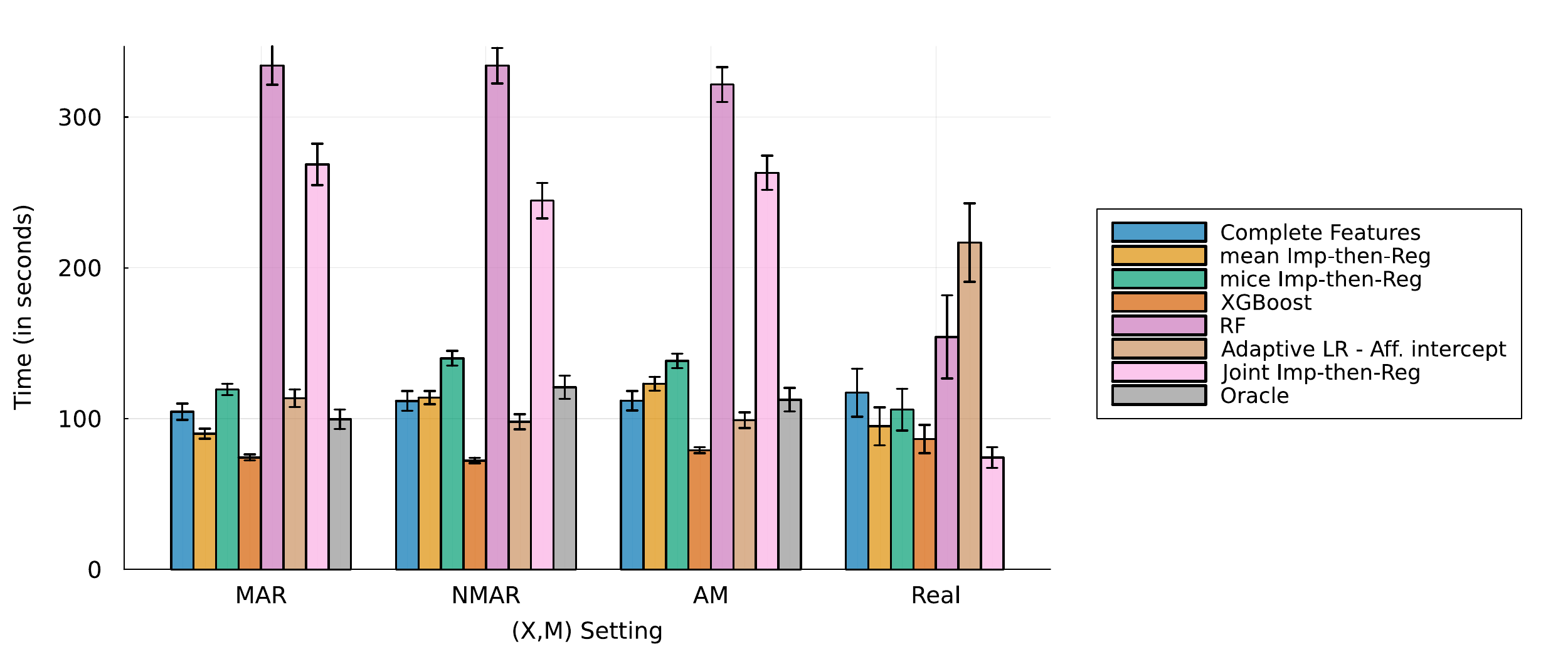}
    \caption{Comparison of adaptive linear regression and joint impute-then-regress methods vs. mean-impute-then-regress on real-world design matrix $\bm{X}$. We also compare to the ``Oracle'' which has access to the fully observed data and the ``Complete Feature'' regression which only uses features which are never missing.}
    \label{fig:realx.comparison.time}
\end{figure}
}

From Figure \ref{fig:realx.comparison}, we conclude that the benefit from our adaptive linear regression models or joint impute-then-regress heuristics is stronger as we deviate from the MAR assumption. To further illustrate this behavior, for each dataset, we count how often the adaptive regression method (resp. joint impute-then-regress heuristic) outperforms mean-impute-then-regress---the second-best performing method---over the 10 random training/test splits and report the distribution of `wins' in Figure~\ref{fig:winrate.adaptive} (resp. Figure \ref{fig:winrate.joint}). 
The trend is clear: as the missingness mechanisms departs further away from the MAR assumption, our adaptive and joint-impute-then-regress models improve more often over mean impute-then-regress. 

\begin{figure}
    \centering
    \begin{subfigure}[t]{.45\textwidth}
    \centering
    \includegraphics[width=\textwidth]{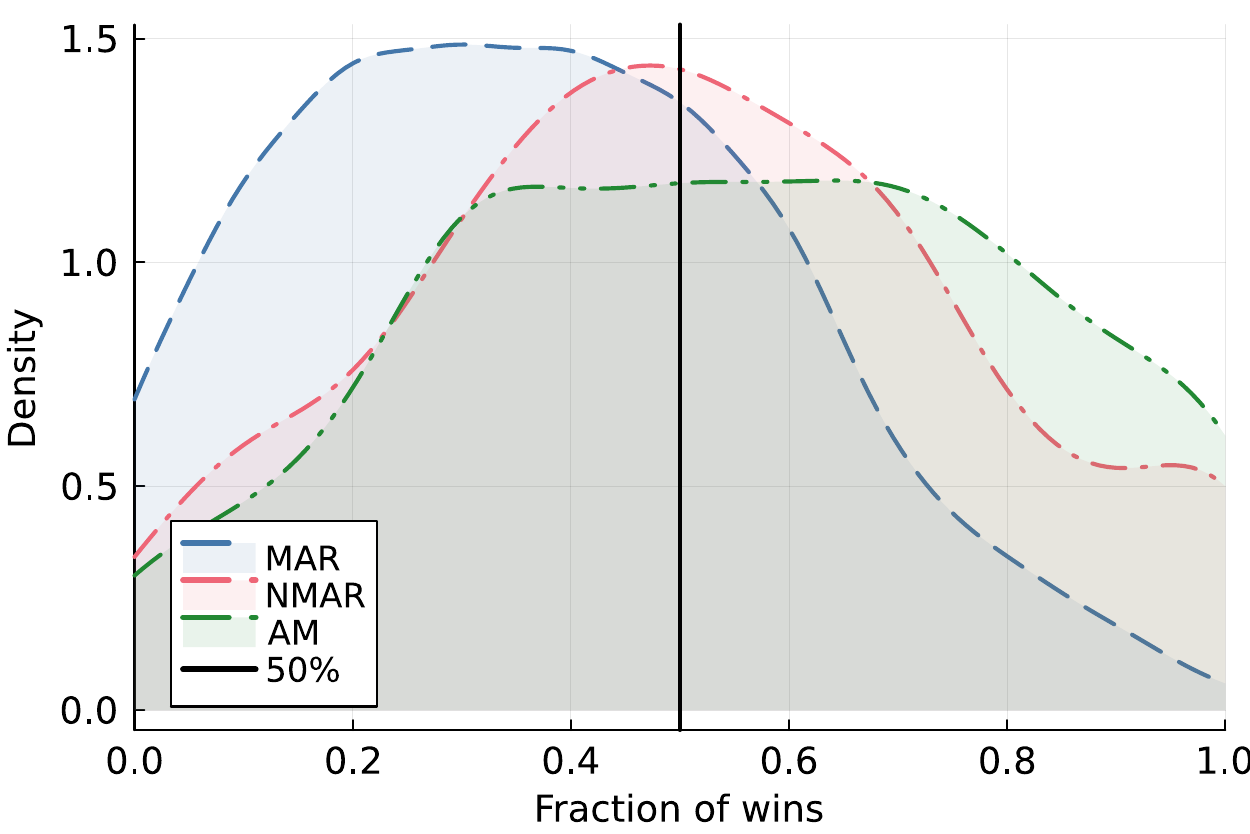}
    \caption{Adaptive linear regression}\label{fig:winrate.adaptive}
    \end{subfigure} %
    \begin{subfigure}[t]{.45\textwidth}
    \centering
    \includegraphics[width=\textwidth]{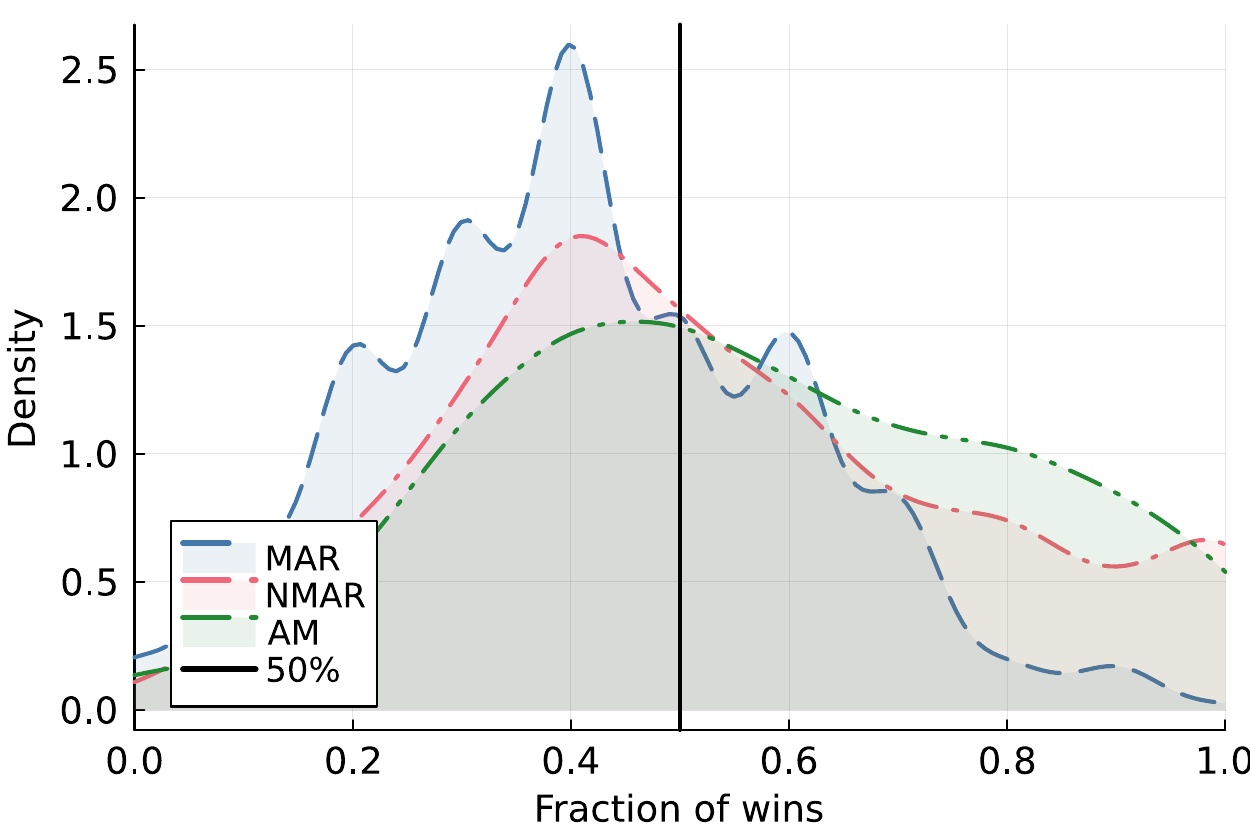}
    \caption{Joint impute-then-regress}\label{fig:winrate.joint}
    \end{subfigure}
    \caption{Frequency of ``wins'' from adaptive models vs. mean-impute-then-regress ones.} \label{fig:winrate}
 \end{figure}

\section{Conclusion}

The literature on missing data in inference is vast; the one on missing data in prediction is comparatively diminutive. Yet, prediction is a distinct and increasingly relevant statistical problem, which can be tackled with new methodologies. 
In this paper, we propose a framework inspired by adaptive optimization to jointly optimize for the imputation of missing data and the downstream predictive task. We demonstrate numerically that our algorithm is competitive with sequential impute-then-regress pipelines when data is missing at random and that it can provide superior predictive power when the true missingness pattern deviates from the MAR assumption. 
In short, our adaptive models have comparable or better performance to standard impute-then-regress methods, while being agnostic to the practically unverifiable missingness mechanism.

\backmatter
\bmhead{Supplementary information}

Proofs of key results and additional numerical results are available in the appendix.


\newpage
\FloatBarrier

\counterwithin{figure}{section}
\counterwithin{table}{section}

\begin{appendices}

\section{Proof of Theorem \ref{prop:finite}} \label{sec:proof.finite}
\begin{lemma}\label{lemma:taylor} For any $t \in \mathbb{Z}_+$, let $w^t_j(\bm{m})$ denote the order-$t$ Taylor expansion of $w^\star_j(\bm{m})$ at $\mathbb{E}[\bm{m}]$. Similarly, denote $f^t(\bm{x}, \bm{m}) = \sum_{j =1}^d w^t_j(\bm{m}) x_j (1-m_j)$. Then, 
\begin{align*}
\left| f^\star(\bm{x}, \bm{m}) - f^t(\bm{x}, \bm{m}) \right| \leq \dfrac{2^{t+1} L}{(t+1)!} \| \bm{m} - \mathbb{E}[\bm{M}] \|^{t+1}.
\end{align*}
\end{lemma}
\begin{proof}[Proof of Lemma \ref{lemma:taylor}]
Let us consider the change of variable $\Tilde{x}_j = x_j (1-m_j)$ and the function (with a slight abuse of notations) $f^\star(\Tilde{\bm{x}}, \bm{m}) = \sum_j w^\star_j(\bm{m}) \Tilde{x}_j$.
If $C$ is a uniform bound on the order-$(t+1)$ derivatives of $f^\star(\Tilde{\bm{x}}, \cdot)$, then Taylor's theorem in multiple variables \citep[][corollary 1]{folland2005higher} yields 
\begin{align*}
  \left| f^\star(\Tilde{\bm x}, \bm{m}) - f^t(\Tilde{\bm x}, \bm{m}) \right| \leq \dfrac{C}{(t+1)!} \| \bm{m} - \mathbb{E}[\bm{M}] \|^{t+1}.  
\end{align*}
Accordingly, if we prove that order-$t$ derivatives of $f^\star(\Tilde{\bm{x}}, \cdot)$ are bounded by $2^t L$, then the result will follow.

Consider a sequence of $t$ indices, $j_1,\dots,j_t$. Consider $j' \in \{1,\dots,d\}$. Recall that 
\begin{align*}
    w^\star_{j'}(\bm{M}) = \sum_{\bm{m} \in \{0,1\}^d} w^\star_{j',m} \prod_{j=1}^d (1 - M_j - m_j + 2 M_j m_j).
\end{align*}
In particular, $w^\star_{j'}(\bm{M})$ is linear in each coordinate $M_j$, all other coordinates being fixed. 
Hence, if the indices $j_1,\dots,j_t$ are not distinct,
\begin{align*}
    \dfrac{\partial^t {w}^\star_{j'}}{\partial M_{j_1}\dots \partial M_{j_t}} (\bm{M}) &= 0,
\end{align*}
and 
$$\dfrac{\partial^t f^\star}{\partial M_{j_1}\dots \partial M_{j_t}} (\Tilde{\bm{X}}, \bm{M}) = 0.$$
as well. On the other hand, if the indices $j_1,\dots,j_t$ are distinct, we have 
\begin{align*}
    \dfrac{\partial^t {w}^\star_{j'}}{\partial M_{j_1}\dots \partial M_{j_t}} (\bm{M}) &= \sum_{m \in \{0,1\}^d} {w}_{j',m}^\star \prod_{\tau =1}^t (2m_{j_\tau}-1) \prod_{i \neq j_1,\dots j_t} (1-M_i-m_i+2M_i m_i).
\end{align*}
Observe that the product on the right-hand side is non-zero only for missingness patterns $\bm{m}$ that satisfy $M_i = m_i$ for all $i \neq j_1,\dots,j_t$. So the sum over all missingness patterns $\bm{m}$ reduces to a sum over $2^t$ terms. Consequently, 
\begin{align*}
    & \dfrac{\partial^t f^\star}{\partial M_{j_1}\dots \partial M_{j_t}} (\Tilde{\bm{X}}, \bm{M}) \\
    & \quad\quad = \sum_{j'=1}^d \sum_{m \in \{0,1\}^d}  {w}_{j',m}^\star \Tilde{X}_{j'} \prod_{\tau =1}^t (2m_{j_\tau}-1) \prod_{i \neq j_1,\dots j_t} (1-M_i-m_i+2M_i m_i) \\
    & \quad\quad = \sum_{m \in \{0,1\}^d}  \left( \sum_{j'=1}^d  {w}_{j',m}^\star \Tilde{X}_{j'} \right) \prod_{\tau =1}^t (2m_{j_\tau}-1) \prod_{i \neq j_1,\dots j_t} (1-M_i-m_i+2M_i m_i).
\end{align*}
In the sum above, at most $2^t$ terms are non-zero. According to Assumption \ref{ass:finite.generative}, $\left| \sum_{j'=1}^d  {w}_{j',m}^\star \Tilde{X}_{j'} \right| \leq L$. Furthermore, $\prod_{\tau =1}^t (2m_{j_\tau}-1) \in \{-1,1\}$ so
\begin{align*}
    \left| \dfrac{\partial^t f^\star}{\partial M_{j_1}\dots \partial M_{j_t}} (\Tilde{\bm{X}}, \bm{M})  \right| \leq 2^t L.
\end{align*}
\end{proof}

\begin{lemma}\label{lemma:taylor.exp} For any $t \in \mathbb{Z}_+$, let $w^t_j(\bm{m})$ denote the order-$t$ Taylor expansion of $w^\star_j(\bm{m})$ at $\mathbb{E}[\bm{m}]$. Similarly, denote $f^t(\bm{x}, \bm{m}) = \sum_{j =1}^d w^t_j(\bm{m}) x_j (1-m_j)$. Then, 
\begin{align*}
\mathbb{E} \left[ \left| f^\star(\bm{X}, \bm{M}) - f^t(\bm{X}, \bm{M}) \right|^2 \right] = \mathcal{O} \left( \left( \dfrac{d}{(t+1)^3} \right)^{t+1} \right).
\end{align*}
\end{lemma}
\begin{proof}[Proof of Lemma \ref{lemma:taylor.exp}]
By Lemma \ref{lemma:taylor}, we have 
\begin{align*}
\mathbb{E} \left[ \left| f^\star(\bm{X}, \bm{M}) - f^t(\bm{X}, \bm{M}) \right|^2 \right] \leq \dfrac{4^{t+1} L^2}{\left(  (t+1)! \right)^2} \mathbb{E}\left[ \| \bm{M} - \mathbb{E}[\bm{M}] \|^{2(t+1)} \right].
\end{align*}
Since $\| \bm{M} - \mathbb{E}[\bm{M}] \|^2 = \sum_{j=1}^d (M_j - \mathbb{E}[M_j])^2 \leq d/4$, $\mathbb{E}\left[ \| \bm{M} - \mathbb{E}[\bm{M}] \|^{2(t+1)} \right] \leq (d/4)^{t+1}$ and we obtain
\begin{align*}
\mathbb{E} \left[ \left| f^\star(\bm{X}, \bm{M}) - f^t(\bm{X}, \bm{M}) \right|^2 \right] \leq \dfrac{d^{t+1} L^2}{\left(  (t+1)! \right)^2}.
\end{align*}
Using Stirling's approximation $(t+1)! \sim \sqrt{2 \pi (t+1)} (\tfrac{t+1}{e})^{t+1}$ yields the result.
\end{proof}

\begin{proof}[Proof of Theorem \ref{prop:finite}]
By decomposing $Y$, 
we obtain 
\begin{align*}
    \mathbb{E}\left[ \left(Y - T_L \hat{f}_n(\bm{X},\bm{M})\right)^2 \right] 
    = \mathbb{E}\left[ \left(f^\star(\bm{X},\bm{M}) - T_L \hat{f}_n(\bm{X},\bm{M})\right)^2 \right] + \mathbb{E}\left[ \varepsilon^2 \right]. 
\end{align*}    
According to \citet[][theorem 11.3]{gyorfi2002distribution}, the first term on the right-hand side can be bounded by 
\begin{align*}
    8\, \inf_{f \in \mathcal{F}} \mathbb{E}\left[ \left(f^\star(\bm{X},\bm{M}) - f(\bm{X},\bm{M})\right)^2 \right] + c \max\{\sigma^2, L\} \dfrac{1+\log n}{n}\,  p(\mathcal{F}),
\end{align*}    
while the second term is bounded by $\sigma^2$. 

To bound $\inf_{f \in \mathcal{F}} \mathbb{E}\left[ \left(f^\star(\bm{X},\bm{M}) - f(\bm{X},\bm{M})\right)^2 \right]$, it is sufficient to bound $\mathbb{E}\left[ \left( f^\star(\bm{X},\bm{M}) - f(\bm{X},\bm{M})\right)^2 \right]$ for some $f \in \mathcal{F}$. In particular, defining $f(\bm{x}, \bm{m})$ as the order-$t$ Taylor expansion (in $\bm{m}$) of $f^\star(\bm{x}, \bm{m})$ at $\bm{m} = \mathbb{E}[\bm{M}]$ yields (Lemma \ref{lemma:taylor.exp} in Appendix \ref{sec:proof.finite})
\begin{align*}
\mathbb{E} \left[ \left| f^\star(\bm{X}, \bm{M}) - f(\bm{X}, \bm{M}) \right|^2 \right] = \mathcal{O} \left( \dfrac{d^{t+1}}{ (t+1)^{3(t+1)}} \right). 
\end{align*}
\end{proof}

\section{Data Description and Evaluation Methodology} \label{sec:data}
In this section, we describe the datasets we used in our numerical experiments, as well as various implementation details. In line with other works in the literature, we conduct some of our experiments on synthetic data, where we have full control over the design matrix $\bm{X}$, the missingness pattern $\bm{M}$, and the signal $Y$. We also contrast the results obtained on these synthetic instances with real world instances from the UCI Machine Learning Repository and the RDatasets Repository\footnote{\url{https://archive.ics.uci.edu} and \url{https://github.com/vincentarelbundock/Rdatasets}}. 

Note that all experiments were performed on a Intel Xeon E5—2690 v4 2.6GHz CPU core using 8 GB RAM. 

\subsection{Synthetic data generation} \label{ssec:data.synthetic}
As in \citet[][section 7]{le2020linear}, we generate a multivariate vector $\bm{X}$ from a multivariate Gaussian with mean $0$ and covariance matrix $\bm{\Sigma} := \bm{B} \bm{B}^\top + \epsilon \mathbb{I}$ where $\bm{B} \in \mathbb{R}^{d \times r}$ with i.i.d. standard Gaussian entries and $\epsilon > 0$ is chosen small enough so that $\bm{\Sigma} \succ \bm{0}$. We fix $d = 10$ and $r=5$ in our experiments. We generate a $n$ observations, $n \in \{40, 60, \dots, 1000\}$, for the training data and $5,000$ observations for the test data.

We then generate signals $Y = f(\bm{X}) + \varepsilon$, where $f(\cdot)$ is a predefined function and $\varepsilon$ is a centered normal random noise whose variance is calibrated to achieve a target signal-to-noise ratio $SNR$. We choose $SNR = 2$ in our experiments. We use functions $f(\cdot)$ of the following forms:
\begin{itemize}
    \item {\bf Linear model: } $f(\bm{x}) = b + \bm{w}^\top \bm{x}$ where $b \sim \mathcal{N}(0,1)$ and $w_j \sim \mathcal{U}([-1,1])$.
    \item {\bf Neural Network (NN) model: } $f(\bm{x})$ corresponds to the output function of a 2-layer neural network with 10 hidden nodes, ReLU activation functions, and random weights and intercept for each node. 
\end{itemize}
We compute $f(\bm{x})$ using a random subset of $k$ out of the $d$ features only, with $k=5$.

Finally, for a given fraction of missing entries $p$, we generate missing entries according to mechanisms
\begin{itemize}
    \item {\bf Missing Completely At Random (MCAR): } For each observation and for each feature $j \in \{1,\dots,d\}$, we sample $M_j \sim Bern(p)$ independently (for each feature and each observation).
    \item {\bf Not Missing At Random (NMAR) - Censoring: } We set $M_j = 1$ whenever the value of $X_j$ is above the $(1-p)$th percentile.
\end{itemize}
For the fraction of missing entries, we consider the different values $p \in \{0.1, 0.2, \dots, 0.8\}$. 

With this methodology, we generate a total of $49$ training sets, with $2$ categories of signal $Y$, $2$ missingness mechanisms, and $8$ proportion of missing entries, i.e., $1,568$ different instances. 

We measure the predictive power of a method in terms of average out-of-sample $R^2$. We use $R^2$, which is a scaled version of the mean squared error, to allow for a fair comparison and aggregation of the results across datasets and generative models.

\subsection{Real-world design matrix} \label{ssec:data.realx}
In addition to synthetic data, we also  assemble a corpus of 63 publicly available datasets with missing data, from the UCI Machine Learning Repository and the RDatasets Repository. Tables \ref{tab:datasets.catmissingonly}, \ref{tab:datasets.nummissingonly}, and \ref{tab:datasets.rest} present summary statistics for the datasets with only categorical features missing, only continuous features missing, and both categorical and continuous features missing respectively. We use the datasets from Tables \ref{tab:datasets.nummissingonly} and \ref{tab:datasets.rest} to compare the performance of different impute-then-regress strategies and our adaptive regression models.

For these datasets, we consider two categories of signal $Y$, real-world and synthetic signals.
\begin{table}[ht]
\footnotesize
\centering
\begin{tabular}{l|ccccc|cc}
Dataset & $n$ & \#features & \#missing cont. & \#missing cat. & $|\mathcal{M}|$ & $d$ & $Y$  \\
\midrule
Ecdat-Males & 4360 & 37 & 0 & 4 & 2 & 38 & cont. \\ 
mushroom & 8124 & 116 & 0 & 4 & 2 & 117 & bin. \\ 
post-operative-patient & 90 & 23 & 0 & 4 & 2 & 24 & bin. \\ 
breast-cancer & 286 & 41 & 0 & 7 & 3 & 43 & bin. \\ 
heart-disease-cleveland & 303 & 28 & 0 & 8 & 3 & 30 & bin. \\ 
COUNT-loomis & 384 & 9 & 0 & 9 & 4 & 12 & cont. \\ 
Zelig-coalition2 & 314 & 24 & 0 & 14 & 2 & 25 & \texttt{NA} \\ 
shuttle-landing-control & 15 & 16 & 0 & 16 & 6 & 21 & bin. \\ 
congressional-voting-records & 435 & 32 & 0 & 32 & 60 & 48 & bin. \\ 
lung-cancer & 32 & 157 & 0 & 33 & 3 & 159 & bin. \\ 
soybean-large & 307 & 98 & 0 & 98 & 8 & 132 & bin. \\ 
\bottomrule
    \end{tabular}
    \caption{Description of the 11 datasets in our library where the features affected by missingness are categorical features only. $n$ denotes the number of observations. The columns `\#features', `\#missing cont.', and `\#missing cat.' report the total number of features, the number of continuous features affected by missingness, and the number of categorical features affected by missingness, respectively. $|\mathcal{M}|$ correspond to the number of unique missingness patterns $\bm{m} \in \{0,1\}^d$ observed, where $d$ is the total number of features after one-hot-encoding of the categorical features. The final column $Y$ indicates whether the dependent variable is binary or continuous (if available).}
    \label{tab:datasets.catmissingonly}
\end{table}
\begin{table}[ht]
\footnotesize
\centering
\begin{tabular}{l|ccccc|cc}
Dataset & $n$ & \#features & \#missing cont. & \#missing cat. & $|\mathcal{M}|$ & $d$ & $Y$  \\
\midrule
auto-mpg & 398 & 13 & 1 & 0 & 2 & 13 & cont. \\ 
breast-cancer-wisconsin-original & 699 & 9 & 1 & 0 & 2 & 9 & bin. \\ 
breast-cancer-wisconsin-prognostic & 198 & 32 & 1 & 0 & 2 & 32 & bin. \\ 
dermatology & 366 & 130 & 1 & 0 & 2 & 130 & cont. \\ 
ggplot2-movies & 58788 & 34 & 1 & 0 & 2 & 34 & \texttt{NA} \\ 
indian-liver-patient & 583 & 11 & 1 & 0 & 2 & 11 & bin. \\ 
rpart-car.test.frame & 60 & 81 & 1 & 0 & 2 & 81 & bin. \\ 
Ecdat-MCAS & 180 & 13 & 2 & 0 & 3 & 13 & cont. \\ 
MASS-Cars93 & 93 & 64 & 2 & 0 & 3 & 64 & cont. \\ 
car-Davis & 200 & 6 & 2 & 0 & 4 & 6 & \texttt{NA} \\ 
car-Freedman & 110 & 4 & 2 & 0 & 2 & 4 & \texttt{NA} \\ 
car-Hartnagel & 37 & 8 & 2 & 0 & 2 & 8 & \texttt{NA} \\ 
datasets-airquality & 153 & 4 & 2 & 0 & 4 & 4 & \texttt{NA} \\ 
mlmRev-Gcsemv & 1905 & 77 & 2 & 0 & 3 & 77 & \texttt{NA} \\ 
MASS-Pima.tr2 & 300 & 7 & 3 & 0 & 6 & 7 & bin. \\ 
Ecdat-RetSchool & 3078 & 37 & 4 & 0 & 8 & 37 & cont. \\ 
arrhythmia & 452 & 391 & 5 & 0 & 7 & 391 & cont. \\ 
boot-neuro & 469 & 6 & 5 & 0 & 9 & 6 & \texttt{NA} \\ 
reshape2-french\_fries & 696 & 9 & 5 & 0 & 4 & 9 & \texttt{NA} \\ 
survival-mgus & 241 & 15 & 5 & 0 & 11 & 15 & bin. \\ 
sem-Tests & 32 & 6 & 6 & 0 & 8 & 6 & \texttt{NA} \\ 
robustbase-ambientNOxCH & 366 & 13 & 13 & 0 & 45 & 13 & \texttt{NA} \\ 
\bottomrule
    \end{tabular}
    \caption{Description of the 22 datasets in our library where the features affected by missingness are numerical features only. $n$ denotes the number of observations. The columns `\#features', `\#missing cont.', and `\#missing cat.' report the total number of features, the number of continuous features affected by missingness, and the number of categorical features affected by missingness, respectively. $|\mathcal{M}|$ correspond to the number of unique missingness patterns $\bm{m} \in \{0,1\}^d$ observed, where $d$ is the total number of features after one-hot-encoding of the categorical features. The final column $Y$ indicates whether the dependent variable is binary or continuous (if available).}
    \label{tab:datasets.nummissingonly}
\end{table}

\begin{table}[ht]
\footnotesize
\centering
\begin{tabular}{l|ccccc|cc}
Dataset & $n$ & \#features & \#missing cont. & \#missing cat. & $|\mathcal{M}|$ & $d$ & $Y$  \\
\midrule
pscl-politicalInformation & 1800 & 1440 & 1 & 1431 & 3 & 1441 & bin. \\ 
car-SLID & 7425 & 8 & 2 & 3 & 8 & 9 & \texttt{NA} \\ 
rpart-stagec & 146 & 15 & 2 & 3 & 4 & 16 & \texttt{NA} \\ 
Ecdat-Schooling & 3010 & 51 & 2 & 8 & 9 & 53 & cont. \\ 
mammographic-mass & 961 & 15 & 2 & 13 & 9 & 18 & bin. \\ 
cluster-plantTraits & 136 & 68 & 2 & 37 & 16 & 85 & \texttt{NA} \\ 
mlmRev-star & 24613 & 122 & 2 & 72 & 19 & 128 & cont. \\ 
car-Chile & 2532 & 14 & 3 & 3 & 7 & 15 & bin. \\ 
heart-disease-hungarian & 294 & 25 & 3 & 17 & 16 & 31 & bin. \\ 
heart-disease-switzerland & 123 & 26 & 3 & 18 & 12 & 32 & bin. \\ 
ggplot2-msleep & 83 & 35 & 3 & 29 & 15 & 37 & \texttt{NA} \\ 
survival-cancer & 228 & 13 & 4 & 4 & 8 & 14 & cont. \\ 
heart-disease-va & 200 & 25 & 4 & 13 & 18 & 30 & bin. \\ 
MASS-survey & 237 & 24 & 4 & 19 & 8 & 29 & \texttt{NA} \\ 
hepatitis & 155 & 32 & 5 & 20 & 21 & 42 & bin. \\ 
automobile & 205 & 69 & 6 & 2 & 7 & 70 & cont. \\ 
echocardiogram & 132 & 8 & 6 & 2 & 13 & 9 & bin. \\ 
thyroid-disease-allbp & 2800 & 52 & 6 & 46 & 25 & 54 & bin. \\ 
thyroid-disease-allhyper & 2800 & 52 & 6 & 46 & 25 & 54 & bin. \\ 
thyroid-disease-allhypo & 2800 & 52 & 6 & 46 & 25 & 54 & bin. \\ 
thyroid-disease-allrep & 2800 & 52 & 6 & 46 & 25 & 54 & bin. \\ 
thyroid-disease-dis & 2800 & 52 & 6 & 46 & 25 & 54 & bin. \\ 
thyroid-disease-sick & 2800 & 52 & 6 & 46 & 25 & 54 & bin. \\ 
survival-pbc & 418 & 27 & 7 & 17 & 8 & 32 & bin. \\ 
thyroid-disease-sick-euthyroid & 3163 & 43 & 7 & 36 & 23 & 44 & bin. \\ 
horse-colic & 300 & 60 & 7 & 52 & 171 & 73 & bin. \\ 
plyr-baseball & 21699 & 296 & 9 & 14 & 18 & 297 & \texttt{NA} \\ 
communities-and-crime & 1994 & 126 & 22 & 3 & 4 & 127 & cont. \\ 
communities-and-crime-2 & 2215 & 129 & 22 & 3 & 4 & 130 & cont. \\ 
wiki4he & 913 & 73 & 44 & 24 & 236 & 78 & cont. \\ 
\bottomrule
    \end{tabular}
    \caption{Description of the 30 datasets in our library where the features affected by missingness are numerical features only. $n$ denotes the number of observations. The columns `\#features', `\#missing cont.', and `\#missing cat.' report the total number of features, the number of continuous features affected by missingness, and the number of categorical features affected by missingness, respectively. $|\mathcal{M}|$ correspond to the number of unique missingness patterns $\bm{m} \in \{0,1\}^d$ observed, where $d$ is the total number of features after one-hot-encoding of the categorical features. The final column $Y$ indicates whether the dependent variable is binary or continuous (if available).}
    \label{tab:datasets.rest}
\end{table}

\FloatBarrier

\subsubsection{Real signal $Y$}\label{sssec:data.realx.realy} 46 out of the 63 datasets had an identified target variable $Y$, which could be continuous or binary. If $Y$ is categorical with more than 1 category, we considered the binary one-vs-all classification task using the first (alphabetical order) category. For regression (resp. classification) tasks, we use the mean squared error (resp. logistic log-likelihood) as the training loss and measure predictive power in terms of $R^2$ (resp. $2 \times AUC - 1$). Again, we choose this measure over mean square error (resp. accuracy or $AUC$) because it is normalized between 0 and 1, and can be more safely compared and aggregated across datasets.

\subsubsection{Synthetic signal $Y$} \label{sssec:data.realx.syny} To generate synthetic signals $Y$, we use the same three generative models as with synthetic data in Section \ref{ssec:data.synthetic}. However, this requires knowledge of the fully observed input matrix, while we only have access to observations with missing entries, $(\bm{o}(\bm{x}^{(i)}, \bm{m}^{(i)}), \bm{m}^{(i)})$, $i=1,\dots,n$. Therefore, we first generate a fully observed version of the data by performing missing data imputation using the \verb|R| package \verb|missForest| \citep{stekhoven2012missforest}, obtaining a new dataset $\{ (\bm{x}_{\text{full}}^{(i)}, \bm{m}^{(i)}) \}_{i\in [n]}$. We use this dataset to generate synthetic signals $Y$, using the three types of signals described in Section \ref{ssec:data.synthetic}: linear, tree, and neural network.

Regarding the relationship between the missingness pattern $\bm{M}$ and the signal $Y$, we consider three mechanisms:
\begin{itemize}
    \item {\bf MAR: } In this setting, we pass $k=\min(10,d)$ coordinates of $\bm{x}_{\text{full}}$ as input to the function $f(\cdot)$. Out of these $k$ features, we explicitly control $k_{missing} \in \{0,\dots,k\}$, the number of features contributing to the signal that are affected by missingness. Hence, in this setting, the resulting response $Y$ depends directly on the covariates $\bm{X}$ but not on the missingness pattern $\bm{M}$. However, we do not control the correlation between $\bm{X}$ and $\bm{M}$ for two reasons: First, they both come from a real-world dataset which might not satisfy the MAR assumption. Second, imputation does induce some correlation between the imputed dataset $\bm{X}_{full}$ and $\bm{M}$.
    \item {\bf NMAR: } In the second setting, in addition to $k=10$ coordinates of $\bm{x}_{\text{full}}$, we also pass $k_{missing}$ coordinates of $\bm{m}$, so that $Y$ is now a function of both $\bm{X}$ and $\bm{M}$.
    \item {\bf Adversarially Missing (AM): } The third setting generates $Y$ in the same way as the {\bf MAR} setting. After $Y$ is generated, however, we reallocate the missingness patterns across observations so as to ensure the data is NMAR. Formally, we consider the observations $(\bm{o}(\bm{x}_{\text{full}}^{(i)}, \bm{m}^{(\sigma_i)}), \bm{m}^{(\sigma_i)}, y^{(\sigma_i)})$, $i\in [n]$, where $\bm \sigma$ is the permutation maximizing the total sum of missing values $\sum_{i=1}^n{\bm{x}_{\text{full}}^{(i)}}^{\top}\bm{m}^{(\sigma_i)}.$ 
\end{itemize}

For each real-world dataset, this methodology generates up to $3 \times 3 \times 11 = 66$ different instances. 

All together, we obtain four experimental settings, with both synthetic and real signals $Y$. They differ in the relationships between the missingness pattern $\bm{M}$, the design matrix $\bm{X}$ and the signal $Y$ as summarized on Figure \ref{fig:exp.designs}.

\subsection{Evaluation pipeline}
In our numerical experiments, we compare a series of approaches, namely 
\begin{itemize}
    \item Regression on the features that are never missing only (complete-feature regression).
    \item Impute-then-regress methods where the imputation step is performed either via mean imputation or using the chained equation method \verb|mice| \citep{buuren2010mice}. We implement these two approaches with a linear, tree, or random forest model for the downstream predictive model. We treat the type of model as an hyper-parameter. We used the default parameter values for number of imputations and number of iterations in  \verb|mice|.
    \item Three variants of the adaptive linear regression framework from Section \ref{sec:adaptive}: Static coefficients with affine intercept, affine coefficients and intercept, finitely adaptive. 
    \item The joint impute-then-regress heuristic presented in Section \ref{sec:extension.adaptive}. Again, we consider three downstream predictive models: linear, tree, random forest.
\end{itemize}
For adaptive linear regression models with affine intercept or affine coefficients, the hyper-parameters are the Lasso penalty $\lambda$, the amount of ridge regularization $\alpha$ (ElasticNet), and the additional lasso penalty for the adaptive coefficients. For the finitely adaptive linear regression model, the only hyper-parameter is the maximum depth of the tree. The other methods can all be applied with linear, tree-based, or random forest predictors. For linear predictors, the hyper-parameters are the Lasso penalty $\lambda$ and the amount of ridge regularization $\alpha$ (ElasticNet). For tree predictors, the hyper-parameter is the maximum depth. For the random forest predictors, the hyper-parameters are the maximum depth of each tree and the number of trees in the forest. 

All hyper-parameters are cross-validated using a $5$-fold cross-validation procedure on the training set. 

We report out-of-sample predictive power on the test set. For synthetic data, the test set consists of $5,000$ observations. For real data, we hold out $30\%$ of the observations as a test set. 

All experiments are replicated $10$ times, with different (random) split into training/test sets between replications.

{\blue 
\section{Additional Numerical Results on Synthetic Data} \label{sec:add.num.syn}
In this section, we report complementary numerical results to Section \ref{ssec:extension.evaluation}.

Figure \ref{fig:syn.ylin} compares the performance of our heuristic for joint impute-then-regress presented in Section~\ref{ssec:extension.heuristic}, with that of mean impute-then-regress and of random forest (RF) trained directly on missing data models using the ``Missing Incorporated in Attribute'' (MIA) method of \citet{twala2008good} on synthetically generated instances where the signal $Y$ is generated from the covariates $\bm{X}$ using a linear model. Figure~\ref{fig:syn.ynn} replicates this analysis on instances where $Y$ depends non-linearly on $\bm{X}$. Our main conclusions remain valid: We observe that our joint impute-then-regress heuristic competes with mean impute-then-regress when data is MCAR but largely outperforms it (by 10\% or more) in the presence of censoring. The main difference from the linear setting comes from the performance of RF. It achieves a similar performance as our joint impute-then-regress strategy in both cases. 
\begin{figure}
    \begin{subfigure}[t]{\columnwidth}
        \centering
        \includegraphics[width=.65\columnwidth]{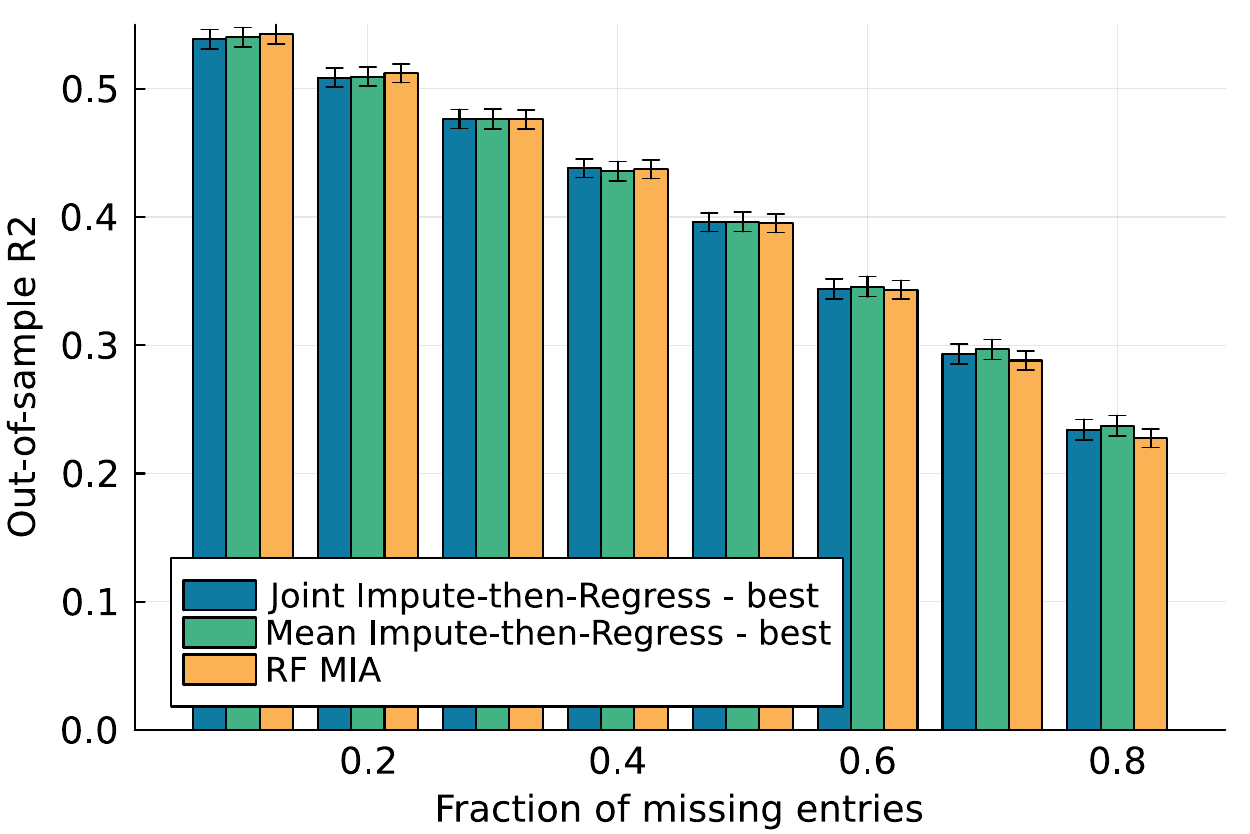}
        \caption{MCAR}
    \end{subfigure} %
    \begin{subfigure}[t]{\columnwidth}
        \centering
        \includegraphics[width=.65\columnwidth]{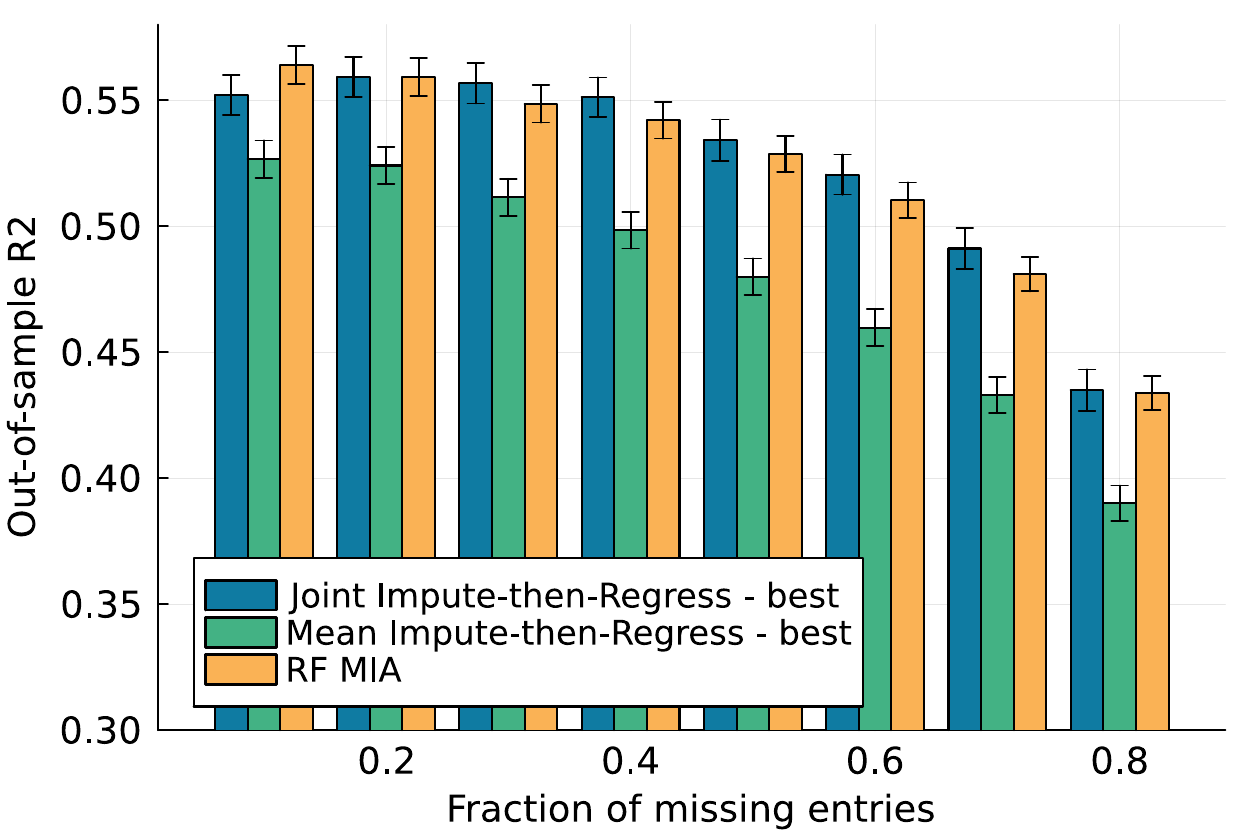}
        \caption{Censoring}
    \end{subfigure} %
    \caption{Out-of-sample $R^2$ of mean-impute-then-regress and joint impute-then-regress (synthetic data, signal produced by a two-layer neural network).}
    \label{fig:syn.ynn}
\end{figure}

Table \ref{tab:syn.res.table} reports the performance of the `best' adaptive LR method, i.e., when the level of adaptivity (affine intercept only, affine, or finite adaptivity) is not fixed a priori but determined via 5-fold cross validation on the training set. To offer a more comprehensive picture, Table \ref{tab:syn.res.table.alr} reports the performance for each level of adaptivity separately. We observe that the two parametric classes (affine intercept only and affine) are superior to the finite adaptive strategy, both in terms of out-of-sample accuracy and computational time. We interpret this observation as being due to the fact that the finitely adaptive model optimizes over a much larger class of models (a class of polynomial weight functions, $\bm{w}(\bm{n})$, of bounded degree, with the maximum degree being related with the maximum depth of the tree, which we cross-validate). 
\begin{table}
    \blue
    \centering \footnotesize
    \begin{tabular}{l|cc|cc}
        Adaptiveness & \multicolumn{2}{c}{MCAR} &  \multicolumn{2}{c}{Censoring} \\
        & Linear & NN & Linear & NN \\
\multicolumn{5}{l}{Out-of-sample $R^2$ (and standard deviation)} \\ 
\midrule
Affine Intercept Only & 0.552 (0.003) & 0.358 (0.003) & \bf 0.647 (0.003) & \bf 0.537 (0.003)  \\ 
Affine & \bf 0.565 (0.003) & \bf 0.366 (0.003) & \bf 0.647 (0.003) & 0.535 (0.027)  \\ 
Finite & 0.538 (0.004) & 0.332 (0.006) & 0.555 (0.027) & -7.083 (5.724)  \\ \\
\multicolumn{5}{l}{Average computational time (in seconds)} \\
\midrule 
Affine Intercept Only & \bf 1.444 & \bf 1.395 & \bf 2.066 & \bf 2.706 \\ 
Affine & 50.051 & 62.778 & 72.061 & 91.474 \\ 
Finite & 11.323 & 11.491 & 9.726 & 11.335 \\ 
    \end{tabular}
    \caption{\blue Performance (out-of-sample $R^2$ and computational time) for each adaptive linear regression approach on synthetic datasets, where $Y$ is generated according to a linear or a neural network model, and the data is either missing completely at random or censored. Results averaged over $50$ training sizes, $8$ missingness levels, and $10$ random training/test splits. }
    \label{tab:syn.res.table.alr}
\end{table}

To confirm this intuition, Figure \ref{fig:syn.n.alr.ylinear} compares the performance of each level of adaptivity, as the number of samples increases $n$, for synthetic data, linear signal $Y$, and a fixed proportion of missing entries $p=0.3$. While the performance of the finite adaptive model competes with that of the other two as $n \rightarrow 1,000$, it is significantly worse for low values of $n$ ($n=40$--$200$) suggesting that it constitutes a much more complex model class that inherently requires more data to be trained. 
\begin{figure}
    \begin{subfigure}[t]{\columnwidth}
        \centering
        \includegraphics[width=.65\columnwidth]{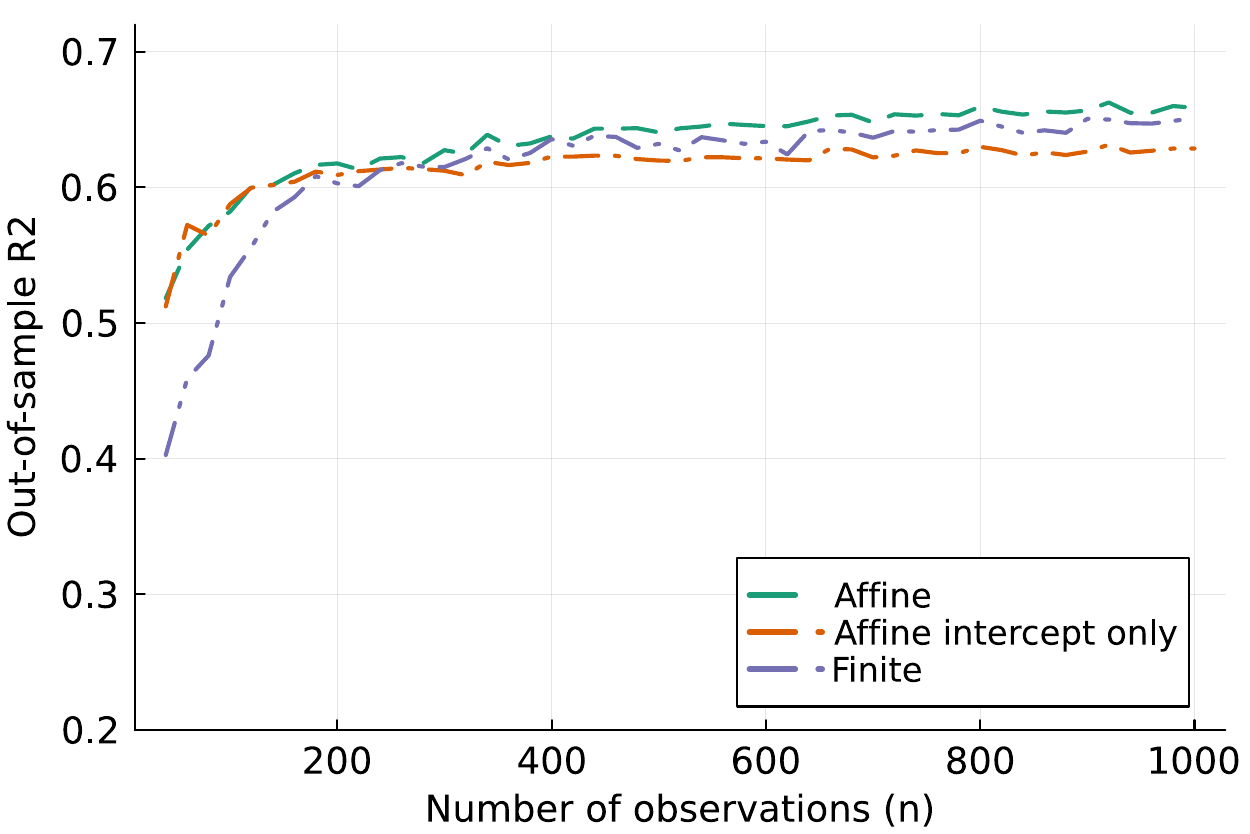}
        \caption{MCAR}
    \end{subfigure} %
    \begin{subfigure}[t]{\columnwidth}
        \centering
        \includegraphics[width=.65\columnwidth]{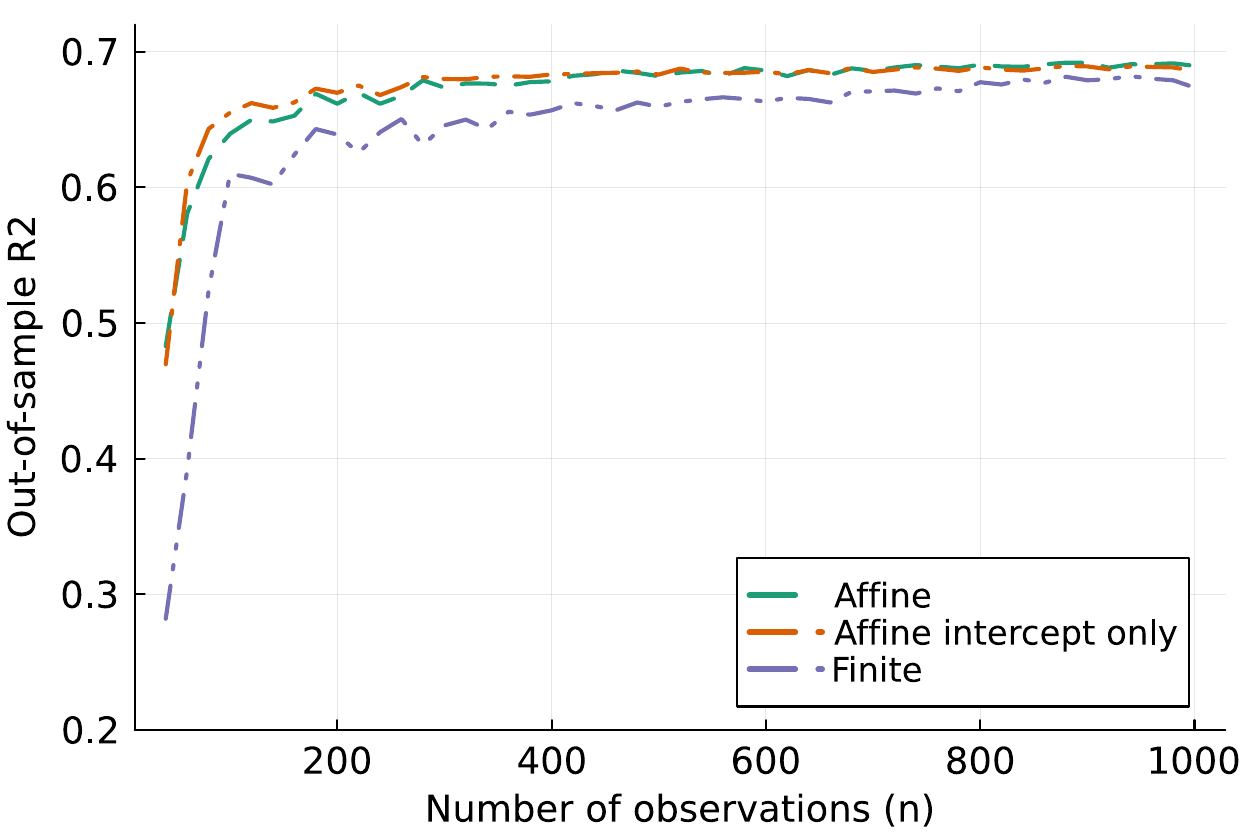}
        \caption{Censoring}
    \end{subfigure} %
    \caption{\blue Out-of-sample $R^2$ of different adaptive linear regression models as the number of observations $n$ increases (synthetic data, signal produced by a linear model, 30\% of missing entries).}
    \label{fig:syn.n.alr.ylinear}
\end{figure}
Looking at the depth of the tree resulting from the greedy procedure Algorithm \ref{alg:finite} (Figure \ref{fig:syn.n.depth.ylinear}), we observe that, when the number of observations is small, the chosen depth is higher. Together with the poor predictive power, this suggests that the finite adaptive model is too complex to be properly calibrated on small sample sizes and is at higher risk of overfitting. 
\begin{figure}
    \begin{subfigure}[t]{\columnwidth}
        \centering
        \includegraphics[width=.65\columnwidth]{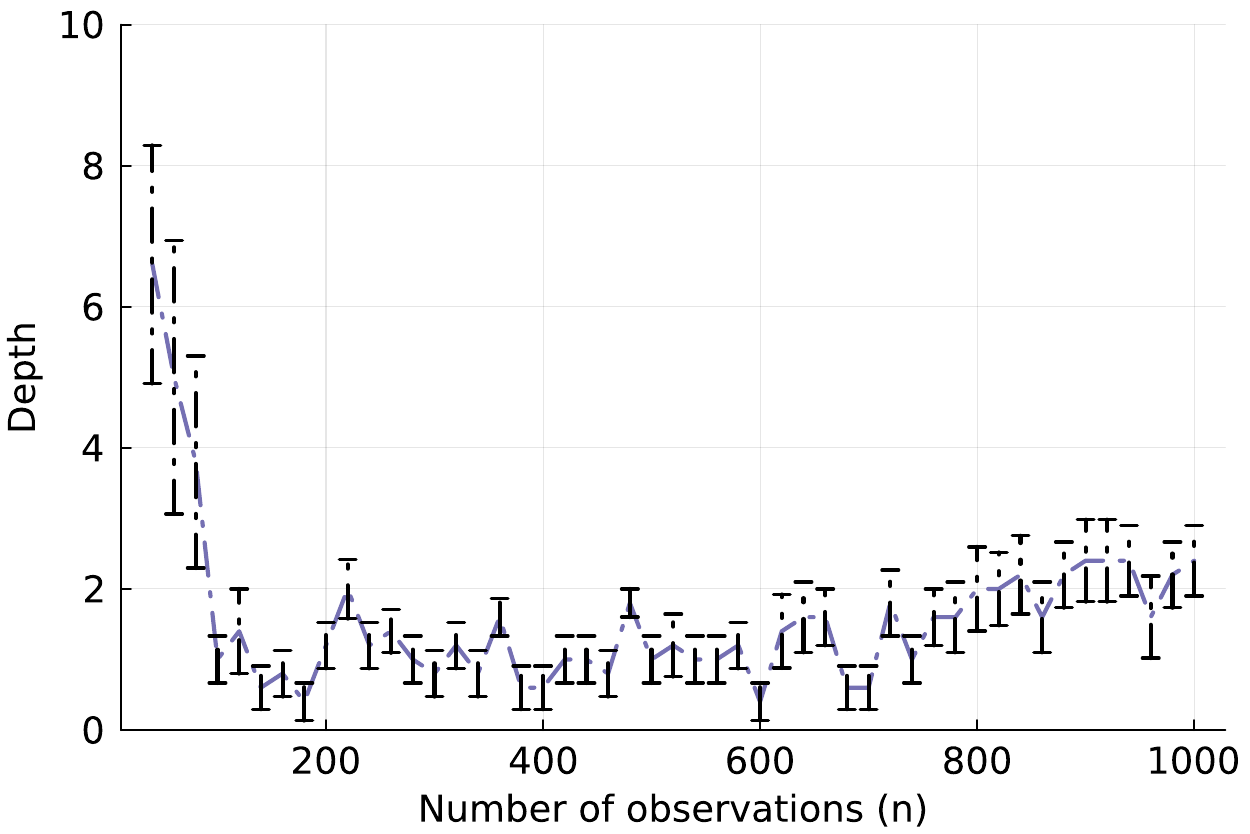}
        \caption{MCAR}
    \end{subfigure} %
    \begin{subfigure}[t]{\columnwidth}
        \centering
        \includegraphics[width=.65\columnwidth]{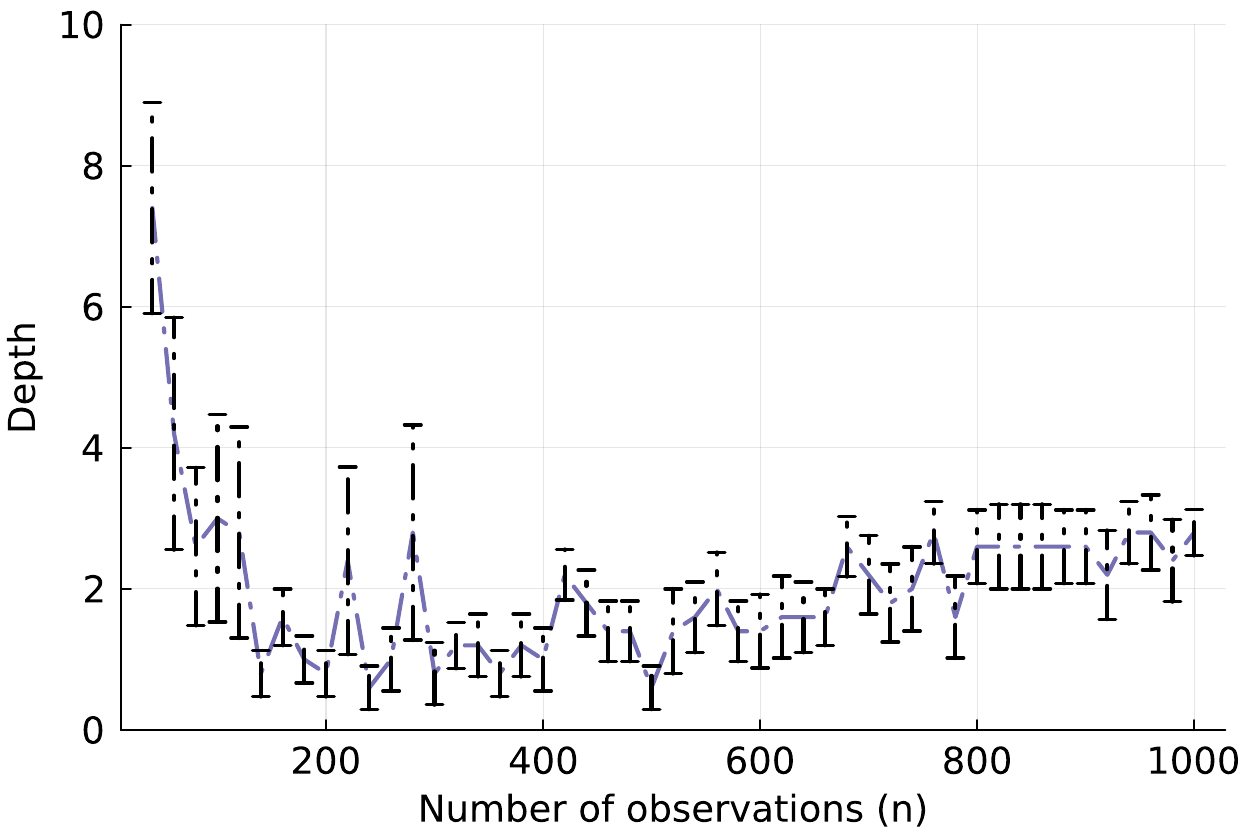}
        \caption{Censoring}
    \end{subfigure} %
    \caption{\blue Tree depth selected for Algorithm \ref{alg:finite} via 5-fold cross-validation on the training set, for the finite adaptive linear regression model, as the number of training observations $n$ increases (synthetic data, signal produced by a linear model, 30\% of missing entries).}
    \label{fig:syn.n.depth.ylinear}
\end{figure}

Table \ref{tab:syn.res.table} compares a set of different methods for handling missing data in terms of predictive power, on the synthetically generated instances. Table \ref{tab:syn.res.table.time} compare them in terms of computational time. 
\begin{table}
    \blue 
    \centering \footnotesize
    \begin{tabular}{l|cc|cc}
        Method & \multicolumn{2}{c}{MCAR} &  \multicolumn{2}{c}{Censoring} \\
        & Linear & NN & Linear & NN \\
\midrule
Adaptive LR - best & 32.266 (0.529) & 35.665 (0.744) & 38.273 (0.764) & 75.185 (0.902)  \\ 
Joint Impute-then-Regress - best & 20.262 (0.796) & 130.274 (1.759) & 15.947 (0.161) & 35.459 (0.888) \\ 
\midrule
Mean Impute-then-Regress - best & 2.601 (0.132) & 16.037 (0.26) & 29.996 (0.352) & 29.977 (0.346)   \\ 
mice Impute-then-Regress - best & 4.27 (0.15) & 19.695 (0.248) & 8.885 (0.132) & 11.716 (0.242)   \\ 
\midrule
CART MIA & 0.197 (0.004) & 0.18 (0.003) & 0.141 (0.004) & 0.144 (0.003)  \\ 
RF MIA & 44.476 (0.849) & 31.612 (0.331) & 26.966 (0.286) & 29.105 (0.315)  \\ 
XGBoost & 87.768 (0.799) & 81.959 (0.743) & 85.738 (0.777) & 89.378 (0.816)  \\ 
    \end{tabular}
    \caption{\blue Average computational time (and standard error), in seconds, for each method on synthetic datasets, where $Y$ is generated according to a linear or a neural network model, and the data is either missing completely at random or censored. Results averaged over $50$ training sizes, $8$ missingness levels, and $10$ random training/test splits. }
    \label{tab:syn.res.table.time}
\end{table}

\FloatBarrier
}
\section{Additional Numerical Results on Real Data}\label{sec:add.num}
In this section, we report complementary numerical results to Section \ref{sec:realdata}.

{\blue 
Figure \ref{fig:validation.alr} compares the average out-of-sample performance for our four adaptive linear regression models. We report computational time in Figure \ref{fig:validation.alr.time}. Although the affine intercept and affine models perform equally in terms of predictive power, the affine intercept model is much cheaper computationally due to the limited number of additional features it involves ($\mathcal{O}(p)$ vs. $\mathcal{O}(p^2)$), hence is the adaptive LR model we would recommend in practice. 
\begin{figure}
    \centering
    \includegraphics[width=.8\columnwidth]{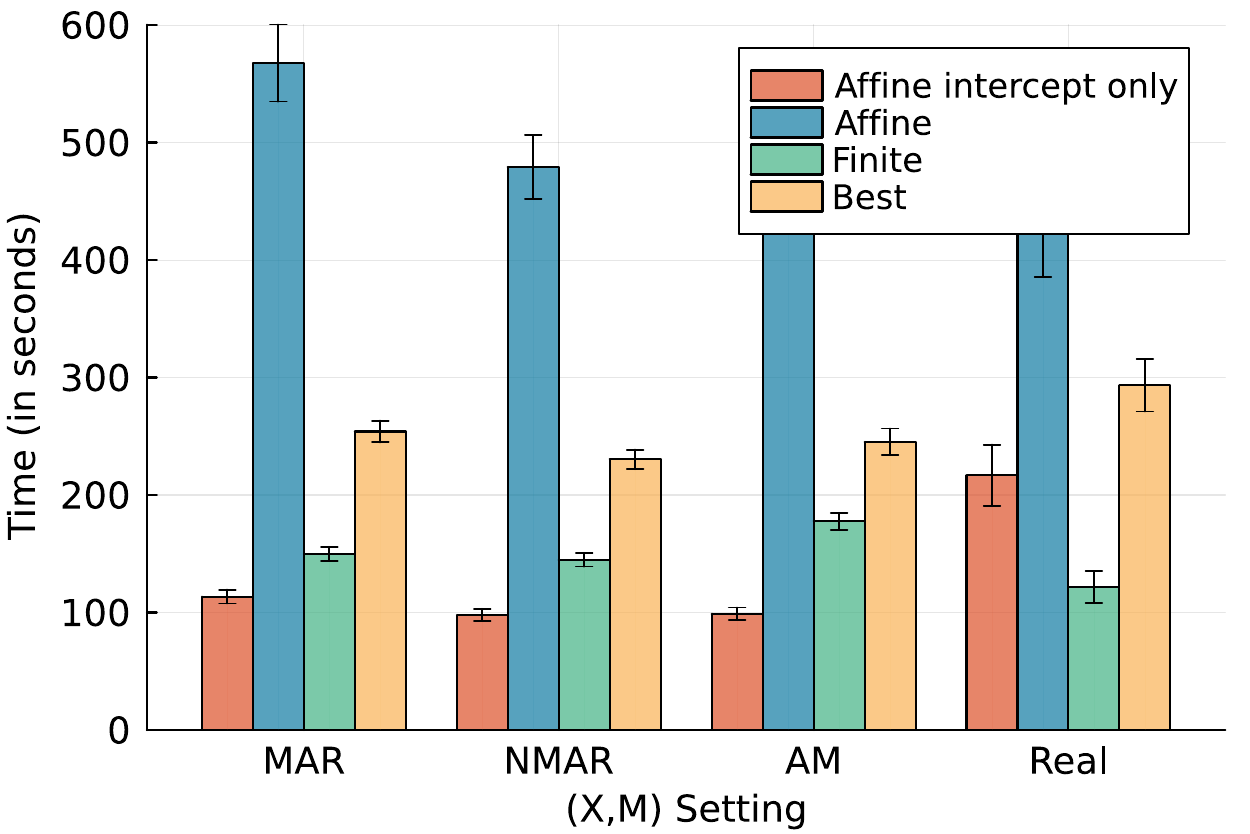}
    \caption{\blue Computational time of adaptive linear regression models on real-world design matrix $\bm{X}$.}
    \label{fig:validation.alr.time}
\end{figure}

Figure \ref{fig:validation.jitr} compares the average out-of-sample performance for our joint impute-then-regress approach, with four different downstream predictive models. We report computational time in Figure \ref{fig:validation.jitr.time}. While random forest (RF) and XGBoost enjoy the same predictive power, training a random forest model (at least in the scikit learn implementation we use) is significantly slower than XGBoost, so we recommend cross-validating the downstream predictor (the `best' variant) between linear and XGBoost only.
\begin{figure}
    \centering
    \includegraphics[width=.8\columnwidth]{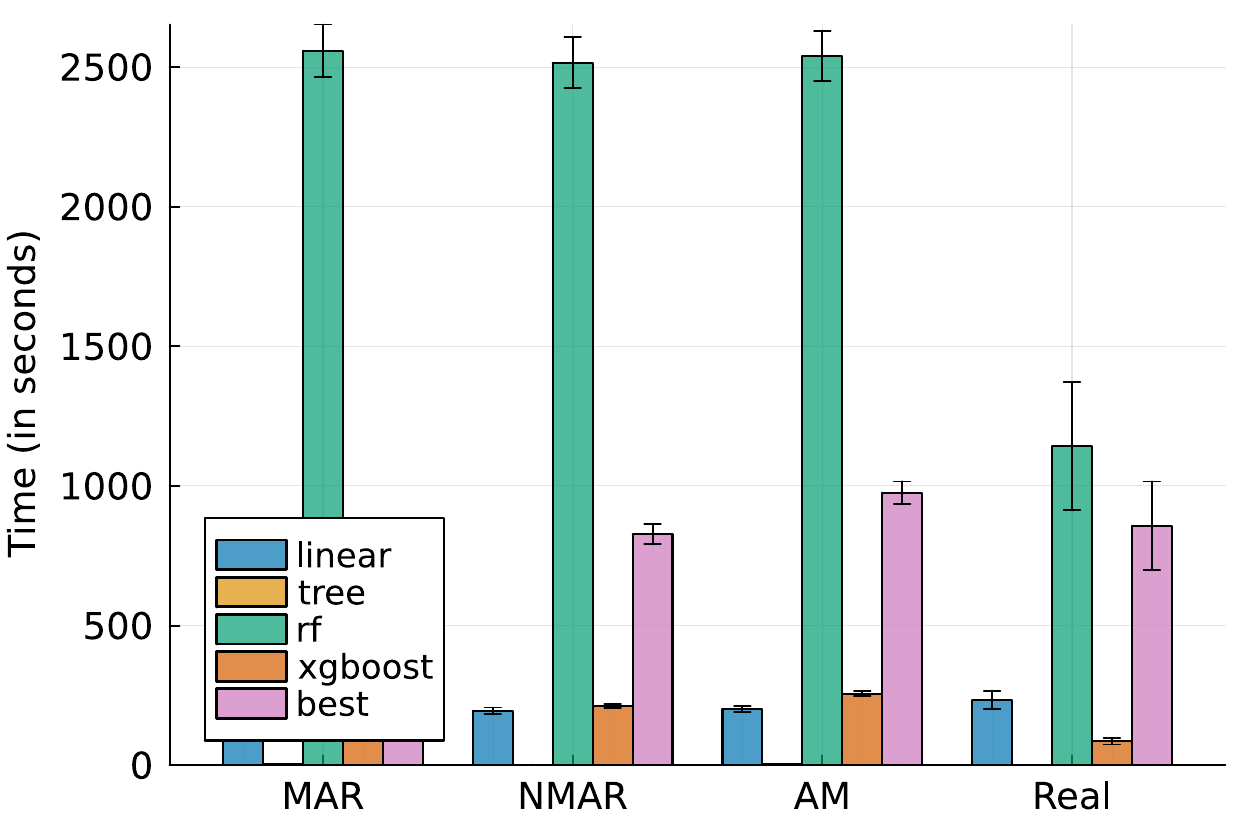}
    \caption{\blue Computational time of joint impute-then-regress method. for different downstream regression function, on real-world design matrix $\bm{X}$.}
    \label{fig:validation.jitr.time}
\end{figure}
}

The out-of-sample accuracy reported in Figure \ref{fig:realx.comparison} 
aggregates, for real signals $Y$, $R^2$ for regression tasks and $2 (AUC - 0.5)$ for classification ones. Although both metrics range from 0 to 1, aggregating results for different metrics could lead to spurious conclusions. As a  sanity check, Figure \ref{fig:realx.realy} compares the performance of mean-impute-then-regress and our adaptive regression models, on the same datasets (real design matrix $\bm{X}$ and real signal $Y$), yet for regression and classification separately. These stratified results confirm a small edge for our adaptive linear models compared with mean impute-then-regress. Furthermore, we observe that our adaptive linear regression models are stronger in regression settings than classification ones.
\begin{figure}
    \centering
    \includegraphics[width=0.7\textwidth]{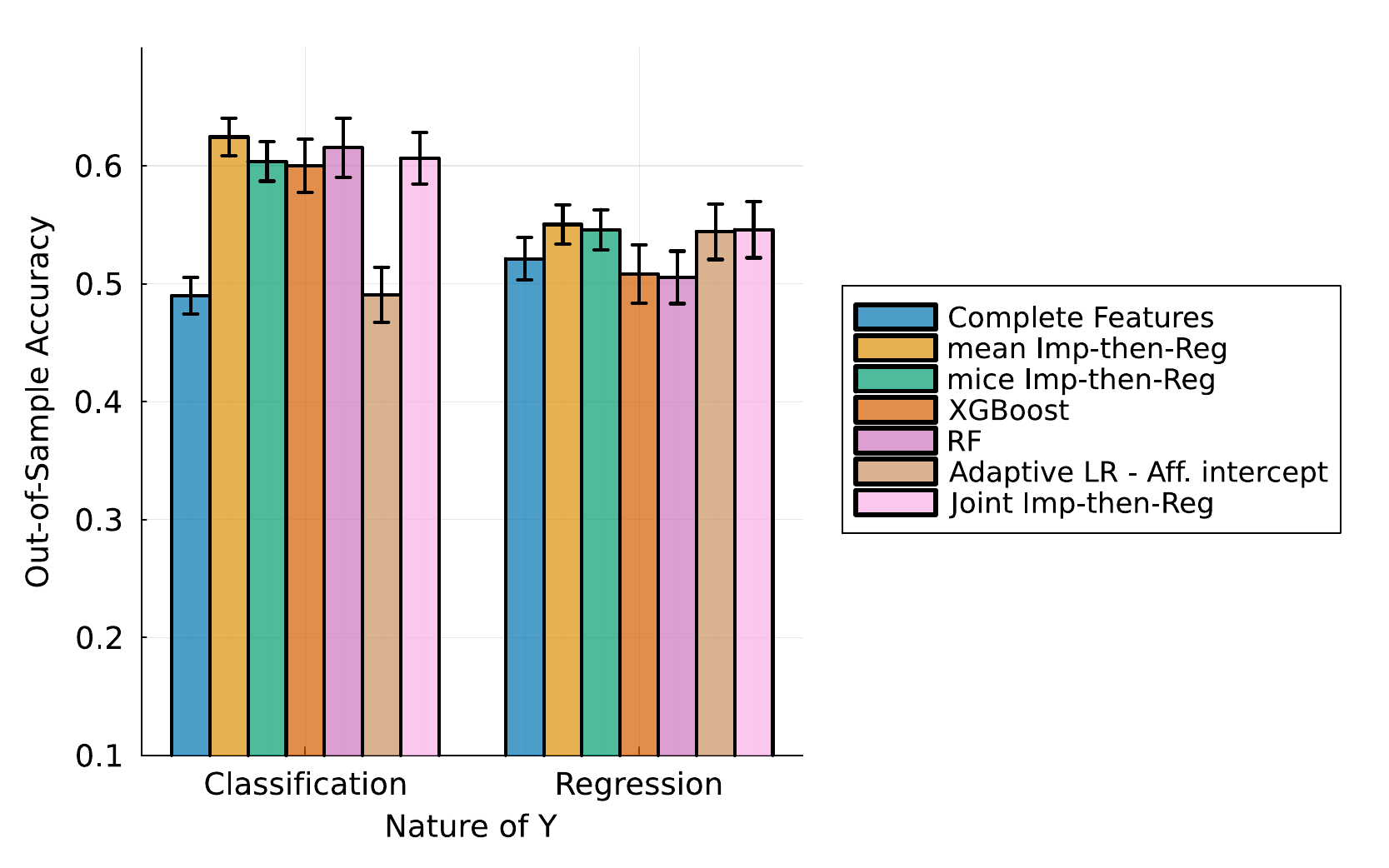}
    \caption{Comparison of adaptive linear regression and joint impute-then-regress methods vs. mean-impute-then-regress on real-world data.}
    \label{fig:realx.realy}
\end{figure}

\FloatBarrier
Figures \ref{fig:realx.comparison} suggests that our adaptive regression models provide a more substantial advantage compared with naive impute-then-regress approaches as the missingness mechanism deviates further from the MAR assumption.
To support this finding, Figure~\ref{fig:adaptive.all} 
reports out-of-sample accuracy for each synthetic signal setting, as a function of the fraction of missing features contributing to the signal. We observe that mean-impute-then-regress methods achieve comparable accuracy as the oracle, adaptive linear regression, and joint-impute-then-regress models in the MAR and NMAR setting. However, they experience a significant drop in accuracy in the AM setting, which grows as the proportion of missing features $k_{missing}/k$ increases.

\begin{figure}
    \centering
    \begin{subfigure}[t]{0.5\textwidth}
        \centering
        \includegraphics[width=\columnwidth]{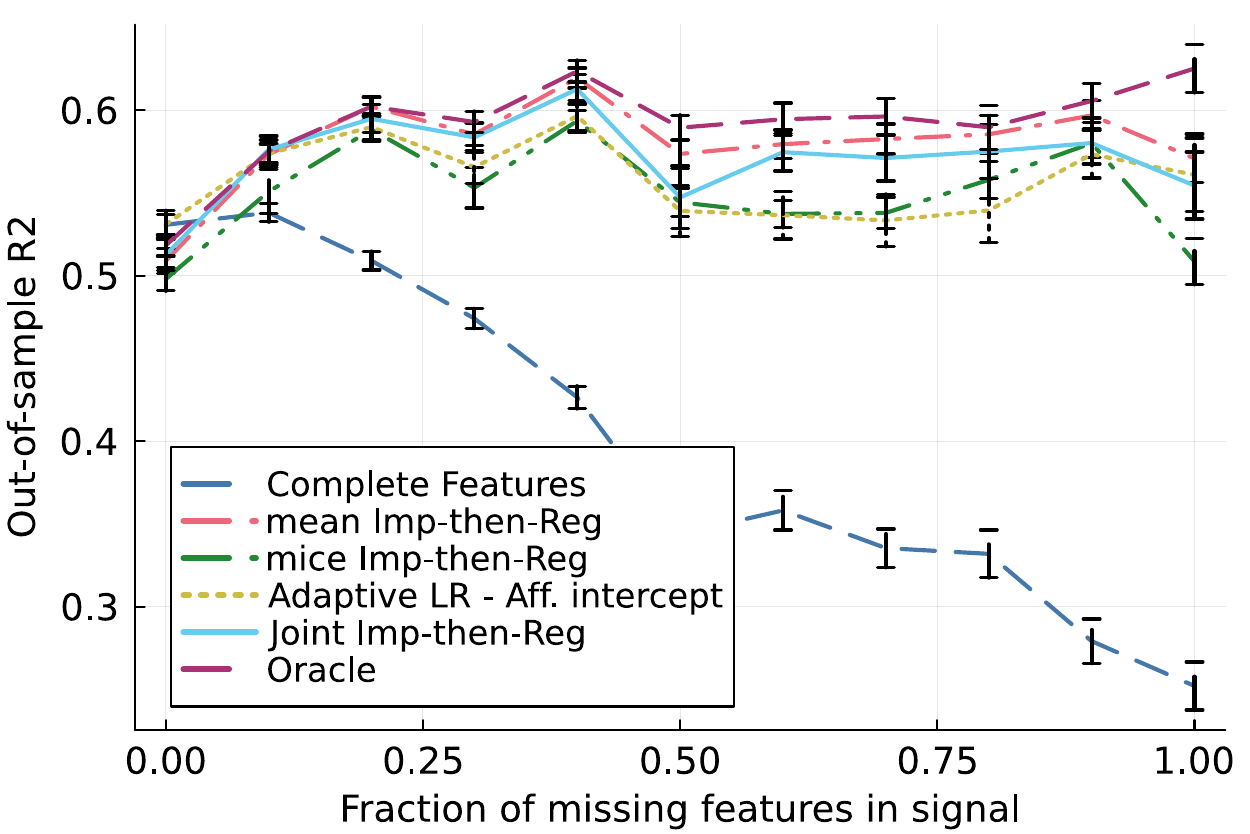}
        \caption{MAR}
    \end{subfigure} %
    
    \begin{subfigure}[t]{0.5\textwidth}
        \centering
        \includegraphics[width=\columnwidth]{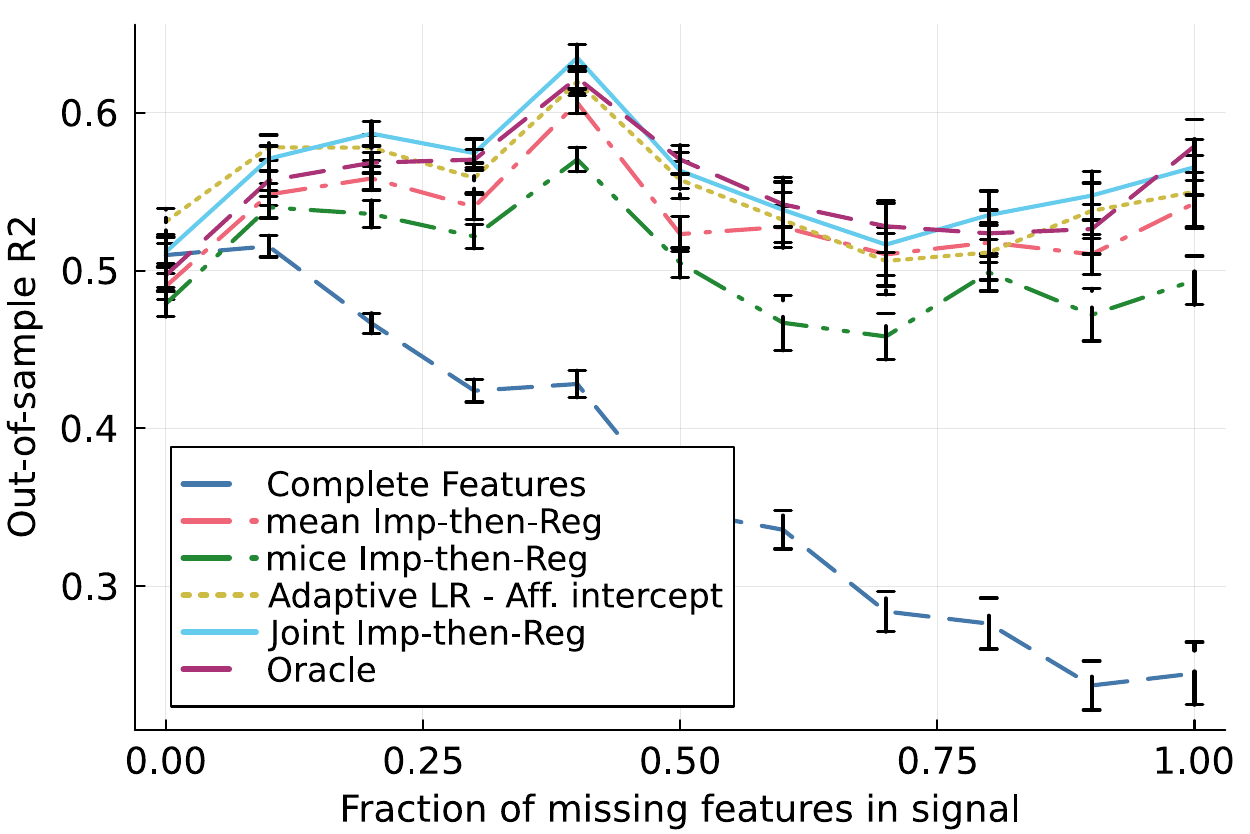}
        \caption{NMAR}
    \end{subfigure} %
    
    \begin{subfigure}[t]{0.5\textwidth}
        \centering
        \includegraphics[width=\columnwidth]{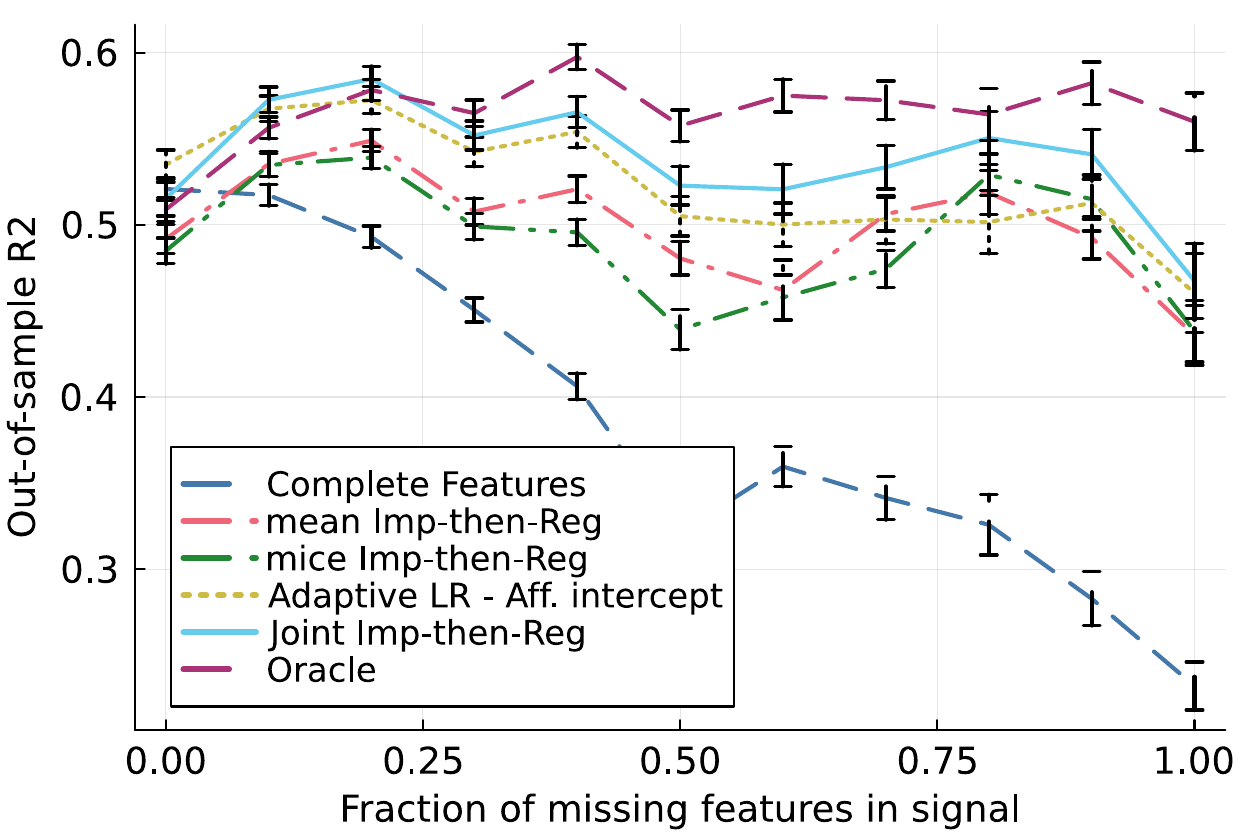}
        \caption{AM}
    \end{subfigure} %
    \caption{Comparison of adaptive regression methods vs. impute-then-regress on synthetic data, as $k_{missing}$ increases. We also compare to two extremes: the ``Oracle'' which has access to the fully observed data and the ``Complete Feature'' regression which regresses on features that are never missing only.}
    \label{fig:adaptive.all}
\end{figure}

\FloatBarrier
We report the data for the average performance (with standard errors) used in Figures \ref{fig:realx.comparison} and \ref{fig:adaptive.all} in Table \ref{tab:fig4.data} and \ref{tab:fig5.data} respectively.
\begin{table}
    \centering
\begin{tabular}{lcccc}
 Method  &    MAR    &   NMAR   &  AM  & Real \\ 
 \midrule
Complete Features & 0.469 (0.004) & 0.461 (0.004) & 0.472 (0.004) & 0.5 (0.017) \\ \midrule 
mean Imp-then-Reg & 0.570 (0.004) & 0.557 (0.004) & 0.532 (0.004) & 0.596 (0.018) \\ 
mice Imp-then-Reg & 0.552 (0.004) & 0.535 (0.004) & 0.521 (0.004) & 0.584 (0.018) \\ \midrule
XGBoost & 0.504 (0.004) & 0.496 (0.004) & 0.487 (0.004) & 0.569 (0.017) \\ 
RF & 0.489 (0.003) & 0.491 (0.003) & 0.480 (0.003) & 0.579 (0.018) \\ 
\midrule 
\blue Adaptive LR - Affine intercept only & 0.559 (0.004) & 0.557 (0.003) & 0.541 (0.003) & 0.508 (0.017) \\ 

\blue Joint Imp-then-Reg   & 0.564 (0.004) & 0.557 (0.004) & 0.546 (0.004) & 0.586 (0.017) \\ 
\midrule
Oracle  & 0.575 (0.004) & 0.573 (0.004) & 0.576 (0.004) &  \\ 
\bottomrule 
 \end{tabular} 
    \caption{Data for Figure \ref{fig:realx.comparison}: Average out-of-sample accuracy (and standard errors) for the different methods, in different missingness settings, for the semi-synthetic and real experiments.}
    \label{tab:fig4.data}
\end{table}

\begin{table}
    \centering \footnotesize
\begin{adjustbox}{angle=90}
\begin{tabular}{lcccccc}
 Fraction missing  &      Complete Features      &     mean Imp-then-Reg     &    mice Imp-then-Reg    &   Adaptive LR  - Aff. intercept  &  Joint Imp-then-Reg  & Oracle \\ 
\multicolumn{6}{l}{Setting:    MAR   } \\ 
 \midrule 
 0.0  & 0.531 (0.006) & 0.509 (0.008) & 0.498 (0.007) & 0.531 (0.008) & 0.513 (0.01) & 0.518 (0.006) \\ 
 0.1  & 0.538 (0.006) & 0.573 (0.006) & 0.551 (0.013) & 0.574 (0.008) & 0.576 (0.008) & 0.575 (0.007) \\ 
 0.2  & 0.509 (0.006) & 0.602 (0.006) & 0.588 (0.006) & 0.59 (0.008) & 0.595 (0.008) & 0.602 (0.006) \\ 
 0.3  & 0.474 (0.006) & 0.585 (0.007) & 0.553 (0.012) & 0.566 (0.01) & 0.584 (0.009) & 0.593 (0.006) \\ 
 0.4  & 0.427 (0.007) & 0.62 (0.006) & 0.594 (0.006) & 0.597 (0.01) & 0.613 (0.009) & 0.624 (0.007) \\ 
 0.5  & 0.344 (0.007) & 0.574 (0.009) & 0.544 (0.009) & 0.539 (0.015) & 0.547 (0.019) & 0.59 (0.008) \\ 
 0.6  & 0.358 (0.012) & 0.58 (0.009) & 0.537 (0.008) & 0.537 (0.014) & 0.575 (0.011) & 0.595 (0.01) \\ 
 0.7  & 0.335 (0.012) & 0.583 (0.009) & 0.538 (0.01) & 0.534 (0.016) & 0.571 (0.014) & 0.596 (0.011) \\ 
 0.8  & 0.332 (0.014) & 0.586 (0.011) & 0.558 (0.011) & 0.54 (0.019) & 0.575 (0.016) & 0.59 (0.013) \\ 
 0.9  & 0.279 (0.013) & 0.597 (0.009) & 0.58 (0.009) & 0.573 (0.014) & 0.58 (0.013) & 0.606 (0.01) \\ 
 1.0  & 0.252 (0.015) & 0.571 (0.015) & 0.509 (0.014) & 0.561 (0.022) & 0.555 (0.02) & 0.625 (0.014) \\ 
\multicolumn{6}{l}{Setting:   NMAR  } \\ 
 \midrule 
 0.0  & 0.51 (0.007) & 0.49 (0.008) & 0.479 (0.008) & 0.531 (0.008) & 0.512 (0.009) & 0.497 (0.008) \\ 
 0.1  & 0.515 (0.007) & 0.548 (0.007) & 0.54 (0.007) & 0.578 (0.008) & 0.571 (0.008) & 0.557 (0.006) \\ 
 0.2  & 0.467 (0.006) & 0.558 (0.007) & 0.536 (0.009) & 0.578 (0.008) & 0.587 (0.008) & 0.568 (0.007) \\ 
 0.3  & 0.424 (0.007) & 0.54 (0.008) & 0.522 (0.008) & 0.559 (0.009) & 0.574 (0.009) & 0.57 (0.006) \\ 
 0.4  & 0.428 (0.008) & 0.606 (0.007) & 0.57 (0.007) & 0.62 (0.009) & 0.635 (0.008) & 0.622 (0.007) \\ 
 0.5  & 0.348 (0.009) & 0.523 (0.011) & 0.505 (0.009) & 0.557 (0.012) & 0.563 (0.011) & 0.57 (0.009) \\ 
 0.6  & 0.336 (0.012) & 0.528 (0.013) & 0.467 (0.017) & 0.532 (0.018) & 0.538 (0.021) & 0.542 (0.014) \\ 
 0.7  & 0.284 (0.013) & 0.51 (0.013) & 0.458 (0.015) & 0.506 (0.021) & 0.516 (0.026) & 0.528 (0.016) \\ 
 0.8  & 0.276 (0.016) & 0.518 (0.013) & 0.499 (0.012) & 0.511 (0.017) & 0.535 (0.015) & 0.524 (0.015) \\ 
 0.9  & 0.237 (0.016) & 0.51 (0.013) & 0.472 (0.016) & 0.538 (0.018) & 0.547 (0.015) & 0.526 (0.016) \\ 
 1.0  & 0.245 (0.02) & 0.543 (0.015) & 0.494 (0.015) & 0.55 (0.023) & 0.565 (0.018) & 0.579 (0.017) \\ 
\multicolumn{6}{l}{Setting:  AM } \\ 
 \midrule 
 0.0  & 0.521 (0.006) & 0.492 (0.008) & 0.485 (0.007) & 0.535 (0.009) & 0.515 (0.01) & 0.509 (0.007) \\ 
 0.1  & 0.517 (0.006) & 0.535 (0.007) & 0.535 (0.007) & 0.567 (0.007) & 0.573 (0.007) & 0.556 (0.006) \\ 
 0.2  & 0.493 (0.006) & 0.549 (0.006) & 0.539 (0.006) & 0.572 (0.008) & 0.585 (0.007) & 0.578 (0.006) \\ 
 0.3  & 0.451 (0.007) & 0.508 (0.008) & 0.499 (0.008) & 0.542 (0.008) & 0.552 (0.009) & 0.565 (0.008) \\ 
 0.4  & 0.406 (0.008) & 0.521 (0.008) & 0.496 (0.008) & 0.554 (0.009) & 0.565 (0.009) & 0.597 (0.007) \\ 
 0.5  & 0.321 (0.009) & 0.48 (0.01) & 0.439 (0.012) & 0.505 (0.011) & 0.523 (0.011) & 0.557 (0.009) \\ 
 0.6  & 0.36 (0.011) & 0.462 (0.017) & 0.458 (0.013) & 0.5 (0.013) & 0.521 (0.014) & 0.575 (0.009) \\ 
 0.7  & 0.341 (0.013) & 0.506 (0.01) & 0.474 (0.011) & 0.503 (0.014) & 0.533 (0.013) & 0.572 (0.011) \\ 
 0.8  & 0.326 (0.017) & 0.519 (0.013) & 0.529 (0.012) & 0.502 (0.018) & 0.55 (0.015) & 0.564 (0.015) \\ 
 0.9  & 0.283 (0.016) & 0.492 (0.012) & 0.515 (0.011) & 0.513 (0.016) & 0.541 (0.014) & 0.582 (0.012) \\ 
 1.0  & 0.232 (0.014) & 0.436 (0.017) & 0.438 (0.018) & 0.46 (0.023) & 0.467 (0.022) & 0.56 (0.017) \\ 
\bottomrule 
 \end{tabular} 
 \end{adjustbox}
    \caption{Data for Figure \ref{fig:adaptive.all}: Average out-of-sample R2 (and standard errors) for the different methods, in different missingness settings, as the proportion of potentially missing covariates contribution to the signal increases.}
    \label{tab:fig5.data}
\end{table}

\FloatBarrier
{Furthermore, } for each unique dataset, we count how often the adaptive linear regression model (resp. joint impute-then-regress model) outperforms mean-impute-then-regress method over 10 random training/test set splits and display the density for this percentage of ``wins'' in Figure \ref{fig:winrate}. The trend is clear: as the missingness mechanisms departs further away from the MAR assumption, adaptive models improve more often over impute-then-regress. 
Comparing these profiles with those obtained on the real $Y$ experiments (Figure \ref{fig:winrate.realY}), we observe that the benefit from our method is not as obvious and systematic as in the semi-synthetic AM case, suggesting that, in our library of real-world datasets, the data is more often missing at random than adversially missing.
\begin{figure}
    \centering
    \begin{subfigure}[t]{.45\textwidth}
    \centering
    \includegraphics[width=\textwidth]{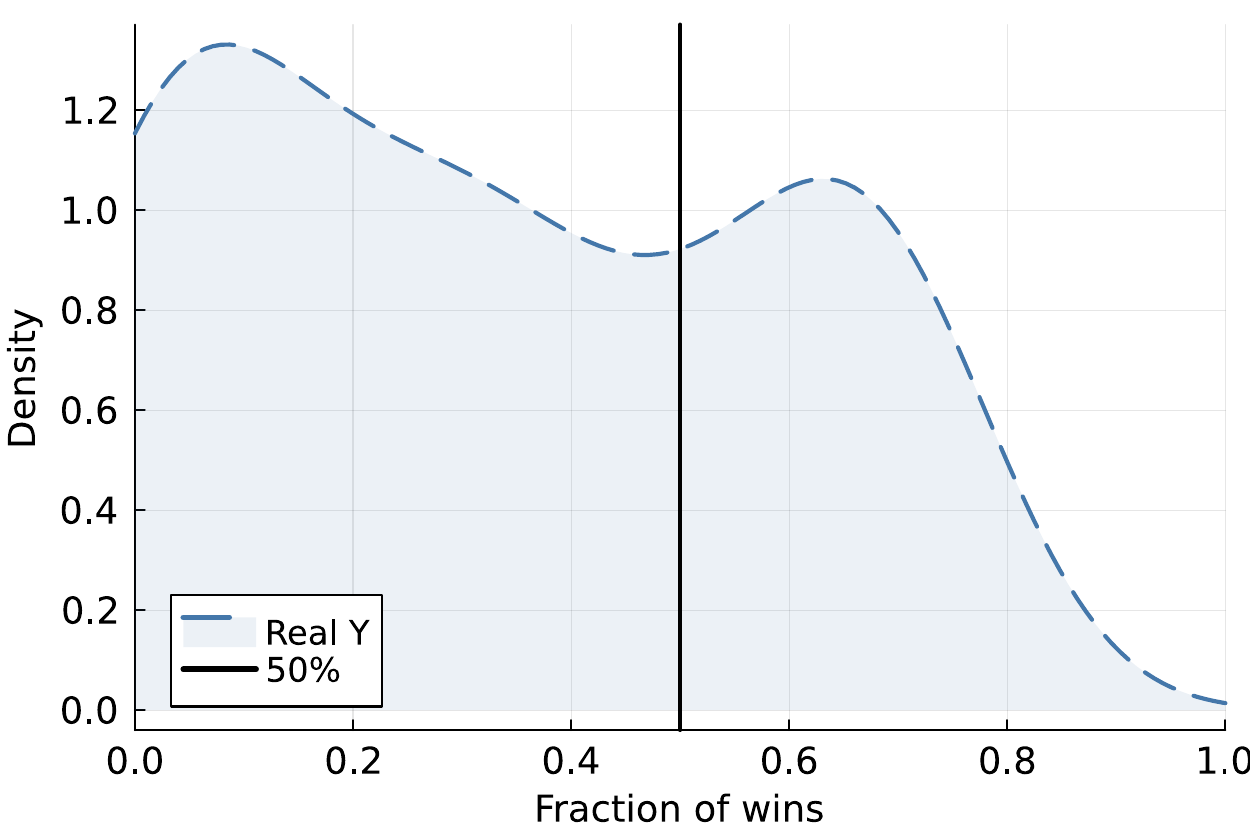}
    \caption{Adaptive linear regression}
    \end{subfigure} %
    \begin{subfigure}[t]{.45\textwidth}
    \centering
    \includegraphics[width=\textwidth]{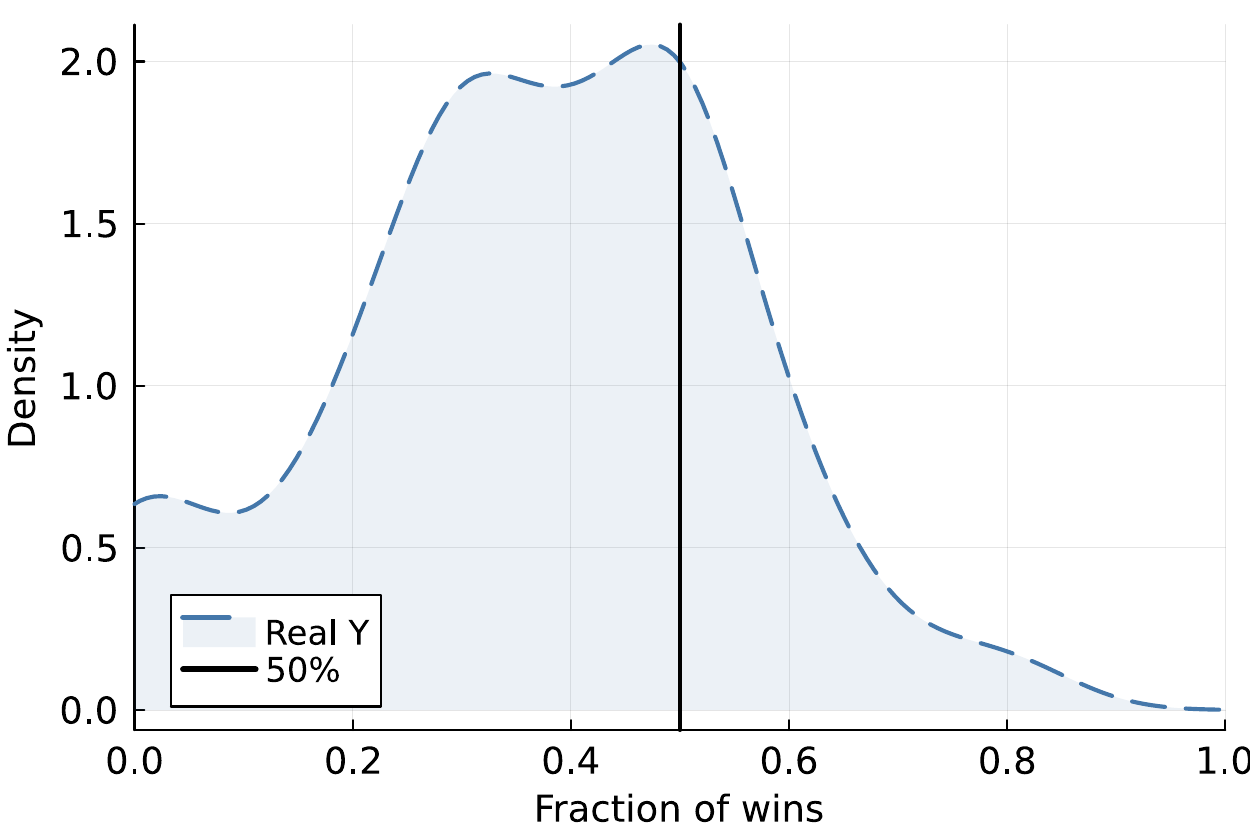}
    \caption{Joint impute-then-regress}
    \end{subfigure}
    \caption{Frequency of ``wins'' from adaptive models vs. mean-impute-then-regress ones.} \label{fig:winrate.realY}
 \end{figure}

\end{appendices}

\bibliography{PHD}

\end{document}